\documentclass{article}

\usepackage{iclr2027_conference,times}

\usepackage[utf8]{inputenc}
\usepackage[T1]{fontenc}
\usepackage{graphicx}
\usepackage{hyperref}
\usepackage{url}
\usepackage{booktabs}
\usepackage{amsfonts}
\usepackage{amsmath}
\usepackage{amssymb}
\usepackage{nicefrac}
\usepackage{microtype}
\usepackage{xcolor}
\usepackage{algorithm}
\usepackage{algorithmic}
\usepackage{subcaption}
\usepackage{multirow}
\usepackage{enumitem}
\usepackage{pifont}
\usepackage{amsthm}
\usepackage{wrapfig}
\usepackage{placeins}

\newtheorem{theorem}{Theorem}
\newtheorem{proposition}[theorem]{Proposition}

\newcommand{\ours}{\textsc{PharmAgent}}

\title{PharmAgent: Constraint-Aware Search with\\Frozen Language Models for Molecular Optimization}

\newcommand{\pharmagentauthorblock}{%
  {\fontsize{11.5pt}{15pt}\selectfont
    \mbox{Nihui Shao\textsuperscript{1,*}},
    \mbox{Guanxing Chen\textsuperscript{1,*,\ensuremath{\dagger}}},
    \mbox{Jilong Shi\textsuperscript{1}},
    \mbox{Zhengyang Bai\textsuperscript{1}},
    \mbox{Haohuai He\textsuperscript{2}},
    \mbox{Zhenchao Tang\textsuperscript{3}},
    \mbox{Qiujie Lv\textsuperscript{4}},
    \mbox{Yu-An Huang\textsuperscript{5,\ensuremath{\dagger}}}, and
    \mbox{Zhi-An Huang\textsuperscript{1,\ensuremath{\dagger}}}\par}
  \vspace{13pt}
  {\fontsize{10.5pt}{13pt}\selectfont
    \textsuperscript{1}\,City University of Hong Kong (Dongguan)\\
    \textsuperscript{2}\,The Hong Kong Polytechnic University\\
    \textsuperscript{3}\,Tencent AI Lab\\
    \textsuperscript{4}\,Zhengzhou University\\
    \textsuperscript{5}\,Northwestern Polytechnical University\par}
  \vspace{9pt}
  {\fontsize{10pt}{12pt}\selectfont
    \textsuperscript{*}\,Equal contribution.\qquad
    \textsuperscript{\ensuremath{\dagger}}\,Corresponding authors.\par}%
}
\author{\pharmagentauthorblock}

\hypersetup{pdfauthor={Nihui Shao, Guanxing Chen, Jilong Shi, Zhengyang Bai, Haohuai He, Zhenchao Tang, Qiujie Lv, Yu-An Huang, Zhi-An Huang}}

\iclrfinalcopy
\usepackage[letterpaper,left=1.25in,right=1.25in,top=1in,bottom=1.05in,
  includeheadfoot,headheight=12pt,headsep=16pt,footskip=24pt]{geometry}
\usepackage{etoolbox}
\renewcommand{\ours}{\mbox{\textsc{PharmAgent}}}
\AtEndEnvironment{abstract}{\clearpage}

\hypersetup{
  hidelinks,
  pdftitle={PharmAgent: Constraint-Aware Search with Frozen Language Models for Molecular Optimization},
  pdfsubject={Molecular optimization with frozen language models},
  pdfkeywords={molecular optimization, constrained optimization, large language models}
}
\date{}

\makeatletter
\renewcommand{\normalsize}{\@setfontsize\normalsize{10pt}{11.5pt}}
\renewcommand{\small}{\@setfontsize\small{10pt}{11.5pt}}
\renewcommand{\footnotesize}{\@setfontsize\footnotesize{10pt}{11.5pt}}
\normalsize
\renewcommand{\@maketitle}{%
  \vbox{\hsize\textwidth\parindent0pt\centering
    \hrule height 2.2pt
    \vskip 10pt
    {\normalfont\bfseries\fontsize{18pt}{22pt}\selectfont\@title\par}
    \vskip 10pt
    \hrule height 0.6pt
    \vskip 17pt
    {\normalsize\normalfont\@author\par}
    \vskip 10pt
  }%
  \lhead{Preprint}%
  \thispagestyle{plain}%
}
\makeatother


\begin{document}
\raggedbottom

\maketitle

%=============================================================================
\begin{abstract}
Molecular optimization must improve target activity and satisfy developability constraints within limited evaluation budgets.
Classical methods require tailored rules or training to incorporate chemical instructions and property feedback.
Frozen language models can condition edits on this information, but need explicit constraint control and relevant experience.
We therefore present \ours{}, a constraint-aware molecular search method driven by adaptive external state.
Its Lagrangian controller translates violations in accepted states into accumulated constraint pressure, keeping this history separate from current property measurements.
Structure-indexed replay complements this feedback with relevant evaluated transitions that guide subsequent proposals.
As a curriculum progressively activates constraints, candidates and the incumbent are compared under the same current objective, and the accepted state determines the next multiplier update.
We derive an exact identity that characterizes how accepted-state violations accumulate in the controller's multipliers.
Across five tasks with five independent runs, \ours{} achieves a summed area under the target-score curves (AUC) of 3.9208 in target-only search, improving over MOLLEO by 37.3\%. With online constraints, it achieves a property-adjusted AUC of 0.7076, improving over the strongest online baseline, ExLLM, by 53.8\%.
These results rank first among all evaluated methods in both target-only and constraint-aware search.
The online comparison covers all five baseline frameworks. The full system leads every ablation variant in target quality, property-adjusted performance, and Pareto hypervolume. All five molecular cases reach feasible final states, documenting target gains and trade-offs.
These results establish \ours{} as a powerful framework for budget-limited molecular optimization, achieving leading aggregate performance through adaptive external search state and frozen language models.
\end{abstract}

\section{Introduction}
\label{sec:intro}
Molecular optimization searches for structures with improved target activity and suitable physicochemical, synthesizability, and absorption, distribution, metabolism, excretion, and toxicity (ADMET) properties~\citep{bickerton2012qed,gao2020synthesizability}.
Under limited budgets, each evaluation should guide productive structural edits~\citep{gao2022sample}.
Graph-GA's stochastic mutation and crossover can spend evaluations on unproductive edits~\citep{jensen2019graph,wang2025molleo}.
REINVENT4 adapts its policy through reward-driven parameter updates~\citep{blaschke2024reinvent4}, while latent-space Bayesian optimization depends on learned representations and valid decoding~\citep{gomez2018automatic}.
These methods require translating chemical instructions and heterogeneous feedback into search rules or training signals~\citep{park2025exllm}.
Large language models (LLMs) can instead condition chemically informed edits directly on task instructions and evaluated experience~\citep{wang2025molleo,park2025exllm,kim2026seismo}.
Freezing their parameters enables adaptation through context without retraining~\citep{yang2024opro,bran2024chemcrow}.

However, proposal quality alone leaves two feedback gaps (Figure~\ref{fig:overview}(a)).
First, target gains can introduce property violations~\citep{gao2020synthesizability}.
Persistent violations call for search priorities that adapt to violation history, beyond fixed penalty weights or post-hoc filtering.
Second, multi-round interaction creates experience, but useful guidance requires edits relevant to the current molecular structure~\citep{shinn2023reflexion,zhao2025cer}.
These gaps motivate dynamic constraint control and structure-aware experience retrieval.

To address both gaps, we propose \ours{}, a constraint-aware molecular search method that adapts external state around frozen language models.
Generator and Reviewer execute the proposal, evaluation, acceptance, and feedback loop with scientific tools and medicinal-chemistry retrieval-augmented generation (MedChem-RAG) (Figure~\ref{fig:overview}(b)).
Three mechanisms organize external state (Figure~\ref{fig:overview}(c)).
A Lagrangian controller accumulates accepted-state violations as historical pressure, separate from current measurements~\citep{bertsekas1999nonlinear}.
Chemical-aware contextual experience replay (Chemical-CER) retrieves structurally related evaluated edits, and a curriculum introduces constraints after target exploration.
Candidates and the incumbent share the current objective, while stored attempts remain distinct from accepted states.
Table~\ref{tab:intro_comparison} compares the evaluated configurations.
\begin{table}[htbp]
\centering\small
\caption{Mechanisms in evaluated configurations (\ding{51} enabled, \ding{55} not enabled). Baseline rows show original configurations. Online adaptations use fixed-penalty (FP) rewards or feasibility-first (FF) rewards, selection, or conversational guidance. Retrieval and cross-task reuse concern optimization attempts. Appendix~\ref{app:online} specifies protocols.}
\label{tab:intro_comparison}
\setlength{\tabcolsep}{4pt}
\begin{tabular}{lcccc}
\toprule
Method & \shortstack{Online\\constraints} & \shortstack{Adaptive\\weights} & \shortstack{Structure-based\\retrieval} & \shortstack{Cross-task\\reuse} \\
\midrule
Graph-GA / REINVENT4 & \ding{55} & \ding{55} & \ding{55} & \ding{55} \\
MOLLEO / ExLLM / SEISMO & \ding{55} & \ding{55} & \ding{55} & \ding{55} \\
Online adaptations & \ding{51} & \ding{55} & \ding{55} & \ding{55} \\
\ours{} & \ding{51} & \ding{51} & \ding{51} & \ding{51} \\
\bottomrule
\end{tabular}
\end{table}

We evaluate target-only search (E1) on practical molecular optimization (PMO)~\citep{gao2022sample} and search with online constraints (E2) on ADMET-PMO, followed by ablations and molecular cases. Our contributions are:
\begin{itemize}[leftmargin=*,itemsep=2pt]
\item A unified molecular search framework that turns property feedback, structurally relevant transitions, and constraint timing into adaptive external decision state for frozen LLMs.
\item An analytical characterization linking multiplier growth to cumulative accepted-state violations, explaining how the controller retains historical constraint pressure.
\item The highest aggregate performance across five tasks and five runs, with a 37.3\% gain in target-score area under the curve (AUC) in E1 and a 53.8\% gain in ADMET-adjusted AUC (AA-AUC) over the strongest online baseline in E2. The full system leads every ablation, and all five cases reach feasibility.
\end{itemize}
\begin{figure}[t]
\centering\includegraphics[width=\linewidth]{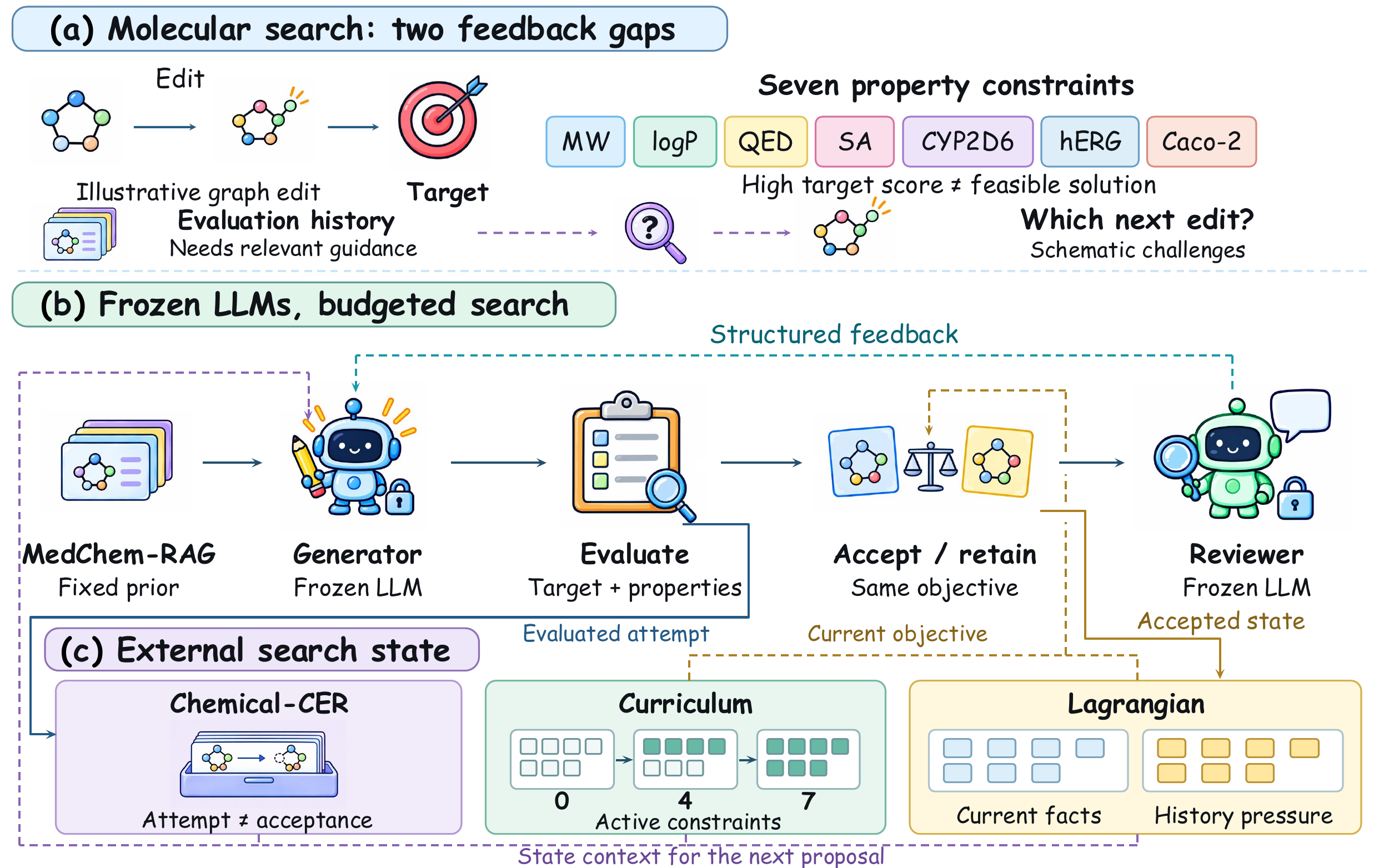}
\caption{Motivation and the \ours{} feedback loop. (a) Two feedback gaps in molecular search. High target scores need not imply feasibility, and past evaluations require relevant retrieval. (b) Frozen Generator--Reviewer workflow for proposal, evaluation, acceptance, and feedback. Candidates and the incumbent share the current objective. (c) External search state comprising Chemical-CER, curriculum, and Lagrangian control. Evaluated attempts enter replay, while accepted states update multipliers. Current measurements remain separate from historical pressure. Property labels denote molecular weight (MW), the logarithm of the octanol--water partition coefficient (logP), quantitative estimate of drug-likeness (QED), synthetic accessibility (SA), cytochrome P450 2D6 (CYP2D6), the human ether-\`a-go-go-related gene potassium channel (hERG), and Caco-2 cell permeability. Icons are schematic. Solid routes show flow or updates, whereas dashed routes show conditioning.}
\label{fig:overview}\label{fig:motivation}\label{fig:architecture}
\end{figure}

\section{Related Work}
\label{sec:related}
\paragraph{Molecular generation with property feedback.}
Multi-property optimization must coordinate conflicting chemical objectives~\citep{bickerton2012qed,gao2020synthesizability}.
InversionGNN uses learned structure--property knowledge and Pareto search~\citep{niu2025inversiongnn}. MolGen incorporates chemical feedback into molecular language modeling~\citep{fang2024molgen}. Uncertainty-aware reinforcement learning (RL) guides diffusion models toward multiple molecular properties~\citep{chen2025uncertainty}.
These approaches connect generative mechanisms to scientific evaluations.
\ours{} adapts measured search state around a frozen model, with explicit online inequalities and a finite target-oracle budget.
\paragraph{Language-model search and memory.}
MOLLEO~\citep{wang2025molleo}, ExLLM~\citep{park2025exllm}, and SEISMO~\citep{kim2026seismo} organize search around populations, experience summaries, and interaction trajectories.
Contextual experience replay~\citep{zhao2025cer} reuses agent interactions without parameter updates. ReAct~\citep{yao2023react} and ChemCrow~\citep{bran2024chemcrow} ground proposals in tool observations.
\ours{} integrates numerical constraint history and structurally relevant transitions directly into proposal-and-acceptance decisions.
\paragraph{Constrained optimization and evaluation.}
Penalty and multiplier methods are classical~\citep{bertsekas1999nonlinear,boyd2004convex}.
Our contribution connects these methods to frozen-LLM molecular search through accepted-state updates, shared-objective comparisons, and separate measurement/history feedback.
Graph-GA~\citep{jensen2019graph} and REINVENT4~\citep{blaschke2024reinvent4} supply non-LLM controls. PMO measures target progress against oracle calls~\citep{gao2022sample}.
ADMET-PMO adds shared RDKit and ADMET-AI evaluations~\citep{rdkit,swanson2024admetai}, supporting both online constraints and post-hoc annotation.

\section{\ours{}: Constraint-Aware Molecular Search}
\label{sec:method}
\subsection{Problem Setup}
\label{sec:problem_setup}
Let $\mathcal X$ be the space of valid molecular structures, $f$ a black-box target oracle, and $g_i$ property evaluators with thresholds $t_i$.
The budget $B$ counts unique target evaluations, including initialization. Neither the evaluators nor the frozen LLM supply gradients.
We seek a high-scoring molecule satisfying $K$ property constraints:
\begin{equation}
\max_{x\in\mathcal X} f(x)\quad\text{s.t.}\quad \tilde g_i(x)\leq\tilde t_i,\quad i=1,\ldots,K.
\label{eq:constraints}
\end{equation}
For an upper bound set $a_i=1$, and for a lower bound set $a_i=-1$, with $\tilde g_i=a_i g_i$ and $\tilde t_i=a_i t_i$.
Normalized violation is $v_i(x)=\max(0,(\tilde g_i(x)-\tilde t_i)/r_i)$ for fixed scales $r_i>0$.
Validity defines $\mathcal X$. Chemically valid structures remain eligible despite property violations.

\subsection{Search Framework}
\label{sec:overview}
The round-$t$ state comprises incumbent $y^{(t)}$, active constraints $\mathcal C_t$, multipliers $\boldsymbol\lambda^{(t)}$, penalty $\rho^{(t)}$, archive $\mathcal D_t$, and Reviewer feedback $R_t$:
\begin{equation}
\begin{aligned}
S_t&=(y^{(t)},\mathcal C_t,\boldsymbol\lambda^{(t)},\rho^{(t)},\mathcal D_t,R_t),\\
\mathcal P_t&\sim q_\theta\bigl(\cdot\mid\operatorname{Encode}(S_t;\mathcal I)\bigr),\qquad\theta\ \text{fixed}.
\end{aligned}
\label{eq:state_interface}
\end{equation}
$\mathcal P_t$ is the proposal batch. $\mathcal I$ supplies molecular representations, property definitions, retrieval features, and fixed chemical context. Measurements adapt $S_t$ while model parameters stay fixed.

The Generator proposes up to five molecular strings in the simplified molecular-input line-entry system (SMILES)~\citep{weininger1988smiles} per request. RDKit~\citep{rdkit} checks validity and novelty after canonicalization. RDKit and ADMET-AI~v1.4.0 supply properties~\citep{swanson2024admetai}.
MedChem-RAG retrieves ChEMBL~\citep{zdrazil2024chembl} matched molecular pairs (MMPs) through fingerprint-based Facebook AI similarity search (FAISS)~\citep{griffen2011mmp,dossetter2013mmp,johnson2019faiss,rogers2010ecfp}.
Only Generator and Reviewer invoke the frozen LLM. The Reviewer returns structured priorities and modification categories. Replay supplies canonical structure pairs and numeric outcomes, excluding historical model prose (Appendix~\ref{app:prompts}).

\subsection{Lagrangian Feedback and Acceptance}
\label{sec:lagrangian}
The controller turns constraint violations into both a numerical acceptance criterion and feedback for subsequent proposals. Classical multiplier and quadratic-penalty terms~\citep{bertsekas1976penalty} define the current objective:
\begin{equation}
\mathcal L_{\mathcal C_t}(x,\boldsymbol\lambda,\rho)
=f(x)-\sum_{i\in\mathcal C_t}\lambda_i v_i(x)
-\frac{\rho}{2}\sum_{i\in\mathcal C_t}v_i(x)^2.
\label{eq:lagrangian}
\end{equation}
An empty active set recovers $f$. For incoming incumbent $y^{(t)}$ and current-objective round-best candidate $c^{(t)}$, acceptance is
\begin{equation}
\bar x^{(t)}=\begin{cases}
c^{(t)},&\mathcal L_{\mathcal C_t}(c^{(t)},\boldsymbol\lambda^{(t)},\rho^{(t)})>
\mathcal L_{\mathcal C_t}(y^{(t)},\boldsymbol\lambda^{(t)},\rho^{(t)}),\\
y^{(t)},&\text{otherwise}.
\end{cases}
\label{eq:accept}
\end{equation}
Both sides use the same active constraints, multipliers, and penalty coefficient. Cached properties rescore the incumbent without another oracle call. This comparison isolates the candidate's improvement under the current objective from changes caused by the controller itself. A score stored before a constraint or multiplier update would conflate these two effects.

After a completed scored round, let $V_{\mathcal C_t}^{(t)}=\sum_{i\in\mathcal C_t}v_i(\bar x^{(t)})$ and update
\begin{equation}
\begin{aligned}
\lambda_i^{(t+1)}&=\lambda_i^{(t)}+\rho^{(t)}v_i(\bar x^{(t)}),\quad i\in\mathcal C_t,\\
\rho^{(t+1)}&=\begin{cases}
\min(\gamma\rho^{(t)},\rho_{\max}),&V_{\mathcal C_t}^{(t)}>0,\\
\rho^{(t)},&V_{\mathcal C_t}^{(t)}=0.
\end{cases}
\end{aligned}
\label{eq:dual_update}
\end{equation}
Multipliers start at zero, making nonnegative projection inactive. Inactive multipliers remain unchanged.
When a proposal is rejected, the retained incumbent drives the update. Thus the accumulated pressure reflects the search state that persists into the next round, while rejected attempts remain available to replay. Empty rounds skip updates.
Active-set expansion resets $\rho$ to $\rho_0>0$ before rescoring, retaining multipliers.

The current-status block reports each active property's value, threshold, direction, signed margin, and violation. The history block reports $\lambda_i$, accumulated from accepted-state violations. These blocks answer different questions. Measurements determine which constraints the current molecule satisfies, while multiplier history identifies constraints that have repeatedly resisted the search.
A feasible molecule can retain a large multiplier because earlier accepted states violated that constraint. This pressure increases the cost of reintroducing the violation in future candidates. Conversely, zero pressure alone does not establish feasibility. The Reviewer receives each candidate's own measurements together with this history, allowing proposal guidance to use both present satisfaction and persistent constraint priorities.

\begin{algorithm}[t]
\caption{\ours{}: external-state search with a frozen proposal model}
\label{alg:pharmagent}
\begin{algorithmic}[1]
\REQUIRE Oracles $f,\{g_i\}$, budget $B$, round cap $T_{\max}$, run-local archive $\mathcal D$
\STATE Score seed pool within $B$. initialize incumbent $x$, seen set $\mathcal S$, spent calls $b$, $\boldsymbol\lambda=0$, $\rho=\rho_0$
\FOR{$t=0,\ldots,T_{\max}-1$ while $b<B$}
\STATE Set $\mathcal C_t$. on expansion reset $\rho$, retain $\boldsymbol\lambda$. rescore $x$
\STATE Retrieve eligible structure-indexed transitions. assemble facts, history, RAG, and feedback
\STATE Generate at most five candidates. verify validity and global novelty
\STATE If none remain, allow one repair request. if still empty, retain context and continue
\STATE Evaluate novel candidates within remaining $B$. update $b,\mathcal S$. select round-best $c$
\STATE Accept $\bar x$ by Eq.~\ref{eq:accept}. invoke Reviewer on scored outcomes and accepted state
\STATE Update multipliers from $\bar x$ by Eq.~\ref{eq:dual_update}. record evaluated transition $(x,c)$
\STATE Update replay and bounded recovery context. set $x\leftarrow\bar x$
\ENDFOR
\RETURN Evaluations, accepted trajectory, dual history, and archive $\mathcal D$
\end{algorithmic}
\end{algorithm}

\subsection{Structure-Indexed Experience Replay}
\label{sec:cer}
Chemical-CER retrieves structurally relevant attempts using source Morgan fingerprints (radius 2, 2048 bits)~\citep{rogers2010ecfp}.
A completed round records incoming incumbent $y^{(t)}$, round-best candidate $c^{(t)}$, target/ADMET changes, and the scalar change under the collection-time objective.
Replay stores the evaluated attempt $(y^{(t)},c^{(t)})$, and the acceptance rule tracks $(y^{(t)},\bar x^{(t)})$. The stored scalar delta retains the collection-time objective.

Retrieve up to $k$ eligible transitions by source-structure proximity:
\begin{equation}
\mathcal E_k^{(t)}=\operatorname{TopK}_{e\in\mathcal D_t^{\mathrm{eligible}}}
\left(-\|\mathbf{fp}_{y^{(t)}}-\mathbf{fp}_{\mathrm{src}}(e)\|_2\right).
\label{eq:retrieval}
\end{equation}
Replay provides target-only information in Phases 1--2 and combined-objective information for all seven constraints in Phase 3, with phase-specific retrieval limits (Appendix~\ref{app:cer}).
Retrieved canonical source/destination pairs and finite numeric deltas condition the Generator alongside fixed RAG context. Each novel proposal is evaluated before acceptance.
The run-local archive persists across tasks.
The CER ablation evaluates the contribution of this replay component.

\subsection{Constraint Scheduling}
\label{sec:admet_scoring}
A fixed target-first curriculum~\citep{xu2025curriculum} activates no ADMET constraints in rounds $[0,20)$, MW/logP/QED/SA in $[20,60)$, and all seven in $[60,\infty)$.
Round indices determine the constraint-activation boundaries.
Expansion adds measurements to evaluation and context, resetting quadratic pressure while retaining multipliers.
Stagnation triggers bounded context recovery with the same activation schedule and acceptance rule (Appendix~\ref{app:compute}).
A continuous sigmoid-based ADMET score supports replay annotation and Pareto evaluation (Eq.~\ref{eq:admet_score}). Hard feasibility requires all seven inequalities, and acceptance uses Eq.~\ref{eq:lagrangian}.

\subsection{Multiplier Dynamics}
\label{sec:theory}
Index only completed updates after full activation, with $V^{(t)}=\sum_i v_i(\bar x^{(t)})$. $\boldsymbol\lambda^{(0)}$ includes prior-phase history.
For any finite $T$, telescoping Eq.~\ref{eq:dual_update} gives the exact identity
\begin{equation}
\sum_{t=0}^{T-1}\rho^{(t)}V^{(t)}=\sum_{i=1}^{K}(\lambda_i^{(T)}-\lambda_i^{(0)}).
\label{eq:violation_accounting}
\end{equation}
Each multiplier accumulates the $\rho$-weighted violations of accepted or retained states. When constraint $i$ is satisfied, its increment is zero while earlier pressure persists.
At fixed measurements and $\rho$, increasing $\lambda_i$ by $\Delta\lambda_i$ lowers a candidate's score by $\Delta\lambda_i v_i(x)$. The accumulated history therefore raises the cost of reintroducing a violation.
Appendix~\ref{app:proof} derives the identity and states a conditional vanishing-violation result under bounded multipliers.

\section{Experiments}
\label{sec:experiments}
\subsection{Tasks, Metrics, and Comparison Protocol}
\label{sec:benchmark}
We evaluate glycogen synthase kinase 3 beta (GSK3$\beta$), c-Jun N-terminal kinase 3 (JNK3), dopamine receptor D2 (DRD2), Sitagliptin multi-property optimization (MPO), and Albuterol similarity, spanning predicted target activity, a multi-property objective, and structural similarity. These tasks test different routes to target improvement under the same seven-property screen.
ADMET-PMO adds the seven shared property constraints illustrated in Figure~\ref{fig:overview}(a), with thresholds in Table~\ref{tab:constraints}.
Target AUC Top-10 measures target progress~\citep{gao2022sample}. AA-AUC averages the running best target score among molecules satisfying all seven constraints. Pareto hypervolume (PHV) measures dominated area in target--continuous-ADMET-score space~\citep{zitzler2003hypervolume}.
PHV uses both feasible and infeasible evaluated molecules.
Definitions appear in Appendix~\ref{app:scales}.

All main and ablation configurations use five independent runs indexed 0--4.
All methods are allowed up to 1,000 unique target evaluations per task/run, including initialization. AUC uses a common 1,000-call horizon.
SEISMO follows its original short-run protocol~\citep{kim2026seismo}. After termination, its final Top-10 and best-feasible scores are carried forward over the remaining horizon.
All LLM methods use MiniMax-M3~\citep{minimax2026m3api}. $\dagger$ denotes independent MOLLEO/ExLLM reimplementations, while SEISMO uses its released prompts and trajectory protocol.
Table rankings compare five-run means.
\ours{} scores 26 seed molecules within budget. Defaults are $\rho_0=1$, $\gamma=1.5$, $\rho_{\max}=100$, and top-5 retrieval.

E1 compares five original baselines. E2 compares six online strategies because Graph-GA uses both fixed-penalty (FP) and feasibility-first (FF) search. FP keeps violation weights constant within each run. FF prioritizes constraint satisfaction over target score. All online baselines use the same seven constraints, ADMET-AI~1.4.0 predictions, and thresholds from their first evaluation. Original trajectories receive separate post-hoc annotation.
Aggregate uncertainty averages the five fixed tasks within each run, then uses five run-level observations.
We report unadjusted 95\% intervals. E2--E1 differences use run-index matching, and cross-method differences use independent-group errors (Appendix~\ref{app:statistics}).

\begin{figure}[t]
\centering\includegraphics[width=\linewidth,trim=0 18bp 0 0bp,clip]{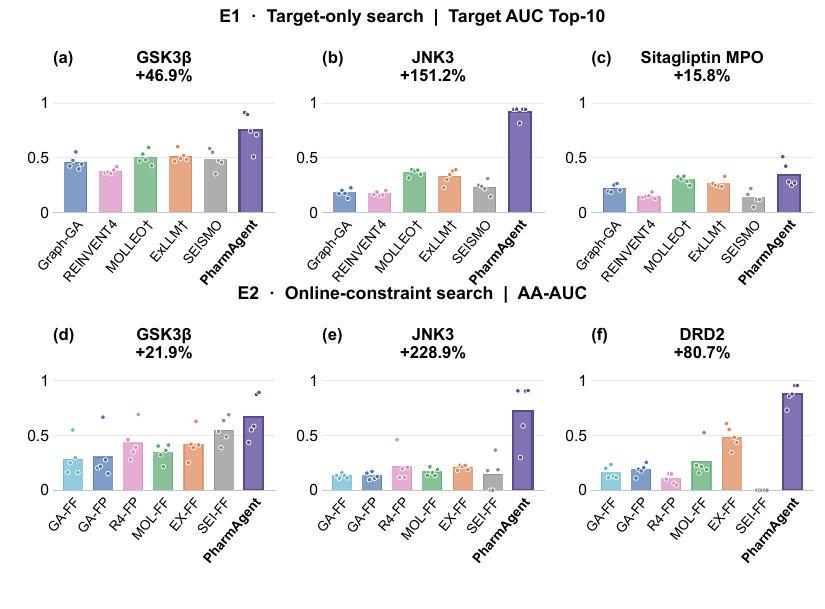}
\caption{Task-level comparisons against (a--c) original baselines in E1 and (d--f) all six online strategies in E2. Bars show means, dots all five runs, and percentages gains over the strongest baseline in each panel. The selected tasks highlight leading means. Full task results appear in Appendix~\ref{app:statistics}. GA, R4, MOL, EX, and SEI abbreviate Graph-GA, REINVENT4, MOLLEO, ExLLM, and SEISMO. FF denotes feasibility-first and FP fixed-penalty search. SEISMO uses terminal-score carry-forward. $\dagger$ marks independent reimplementations.}
\label{fig:main_results}
\end{figure}

\subsection{E1 Target-Only Search Capability}
\label{sec:e1}
\begin{table}[htbp]
\caption{Standard PMO target AUC Top-10 (higher is better), averaged over five runs. Sum adds the five task means. All LLM methods use MiniMax-M3. $\dagger$ and SEISMO's original short-run protocol with terminal-score carry-forward are defined in Section~\ref{sec:benchmark}. Best/second-best means are bold/underlined.}
\label{tab:standard_pmo}
\centering
\small
\begin{tabular}{lcccccc}
\toprule
Method & GSK3$\beta$ & JNK3 & DRD2 & Sitagliptin & Albuterol & Sum \\
\midrule
Graph-GA & 0.4600 & 0.1814 & 0.6939 & 0.2213 & 0.4423 & 1.9988 \\
REINVENT4 & 0.3792 & 0.1742 & 0.6917 & 0.1513 & 0.3947 & 1.7911 \\
MOLLEO$^{\dagger}$ & 0.5028 & \underline{0.3658} & 0.8041 & \underline{0.2999} & 0.8840 & \underline{2.8565} \\
ExLLM$^{\dagger}$ & \underline{0.5150} & 0.3302 & 0.7981 & 0.2657 & 0.8860 & 2.7950 \\
SEISMO & 0.4850 & 0.2326 & \textbf{0.9756} & 0.1344 & \textbf{0.9791} & 2.8067 \\
\ours{} & \textbf{0.7564} & \textbf{0.9186} & \underline{0.9494} & \textbf{0.3472} & \underline{0.9492} & \textbf{3.9208} \\
\bottomrule
\end{tabular}
\end{table}
E1 disables ADMET guidance and Lagrangian penalties, retaining target-only scoring and validity checks.
\ours{} achieves the highest target-AUC sum of 3.9208, surpassing MOLLEO$^\dagger$ at 2.8565 by 37.3\%.
It leads GSK3$\beta$, JNK3, and Sitagliptin and ranks second on DRD2 and Albuterol.
The aggregate target-AUC difference from each of the four original full-budget baselines has a positive 95\% interval (Appendix~\ref{app:statistics}).

\subsection{E2 Search with Online Constraints}
\label{sec:e2}
\begin{table}[htbp]
\caption{Online-constraint results across five tasks and five runs. Target AUC Sum adds task means. The other metrics average tasks and runs. All methods receive online constraints. FF denotes feasibility-first and FP fixed-penalty search. Best/second-best means are bold/underlined. Figure~\ref{fig:online_controls} reports uncertainty.}
\label{tab:admet_pmo}
\centering\small
\setlength{\tabcolsep}{4pt}
\begin{tabular}{lccccc}\toprule
Method & Target AUC Sum & AA-AUC & PHV & Top-1 & Top-10 \\\midrule
Graph-GA (FF) & 1.0675 & 0.2536 & 0.4181 & 0.4375 & 0.3219 \\
Graph-GA (FP) & 1.1809 & 0.2711 & 0.4630 & 0.4837 & 0.3575 \\
REINVENT4 (FP) & 1.8337 & 0.2859 & 0.5622 & 0.6221 & 0.4663 \\
MOLLEO$^\dagger$ (FF) & 2.0465 & 0.4006 & 0.6430 & \underline{0.6788} & 0.5864 \\
ExLLM$^\dagger$ (FF) & 2.2732 & \underline{0.4602} & \underline{0.6514} & 0.6647 & \underline{0.6352} \\
SEISMO (FF) & \underline{2.9319} & 0.3838 & 0.5795 & 0.6405 & 0.5978 \\
\ours{} (E2) & \textbf{3.6936} & \textbf{0.7076} & \textbf{0.8248} & \textbf{0.8514} & \textbf{0.8305} \\
\bottomrule\end{tabular}
\end{table}
\ours{} leads all six online strategies on every aggregate metric in Table~\ref{tab:admet_pmo}.
Its target-AUC sum of 3.6936 exceeds the strongest baseline, SEISMO FF (2.9319), by 26.0\%.

For AA-AUC and PHV, ExLLM FF is the strongest baseline. \ours{} improves its 0.4602 AA-AUC to 0.7076 (53.8\%) and its 0.6514 PHV to 0.8248 (26.6\%).
The AA-AUC difference is 0.2474, with an unadjusted 95\% interval of $[0.1355,0.3594]$.

\begin{figure}[htbp]
\centering\includegraphics[width=\linewidth]{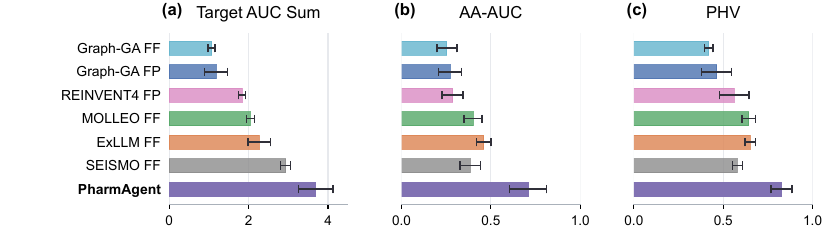}
\caption{Online comparison across five tasks and five runs. (a--c) Target-AUC sum, AA-AUC, and PHV for six baseline strategies and PharmAgent E2. Feasibility-first (FF) search prioritizes constraint satisfaction. Fixed-penalty (FP) search penalizes total normalized violation with a weight fixed within each run. Bars show means with unadjusted 95\% intervals after within-run task aggregation. All methods have the same budget allowance. SEISMO retains its short-run protocol and terminal-score carry-forward.}
\label{fig:online_controls}
\end{figure}
\ours{} leads on target-AUC sum (Figure~\ref{fig:online_controls}(a)), AA-AUC (Figure~\ref{fig:online_controls}(b)), and PHV (Figure~\ref{fig:online_controls}(c)), with positive difference intervals against every online strategy on all three metrics.
With shared online constraint measurements, \ours{} leads target progress, feasible-solution progress, and Pareto coverage.
Its target-AUC advantage extends to all five tasks, and it also leads AA-AUC on GSK3$\beta$, JNK3, and DRD2.

\paragraph{Impact of online constraint control.}
Post-hoc evaluation of E1's unchanged trajectories gives AA-AUC 0.5328, versus E2's 0.7076, a 32.8\% improvement with difference interval $[0.0511,0.2987]$.
All five run-level aggregate differences are positive.
Figure~\ref{fig:main_results}(a--c) highlights target-search gains over the original baselines, while Figure~\ref{fig:main_results}(d--f) shows AA-AUC gains over their online adaptations. Appendix~\ref{app:statistics} provides budgeted progress in Figure~\ref{fig:search_evidence}(a) and all five E2--E1 task effects in Figure~\ref{fig:search_evidence}(b). Task-level AA-AUC intervals exclude zero for DRD2, Sitagliptin, and Albuterol.
Online feedback thus turns strong target search into stronger feasible-solution search.

\begin{figure}[t]
\centering\includegraphics[width=\linewidth,trim=0 3bp 0 17bp,clip]{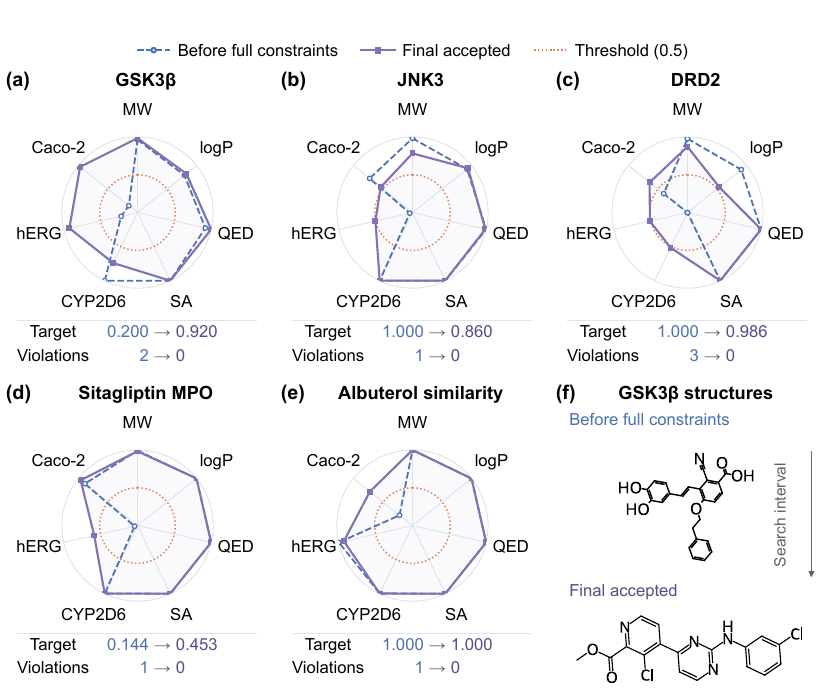}
\caption{Five molecular cases (run~0, before full constraint activation at round~59 versus final accepted state). (a--e) Sigmoid desirability profiles under all seven requirements. Outward is favorable, and the 0.5 ring marks thresholds. CYP2D6 denotes inhibition. Target scores and violation counts accompany each profile. (f) Recorded GSK3$\beta$ structures with an arrow spanning subsequent rounds. ADMET endpoints are predictions.}
\label{fig:cases}
\end{figure}

\subsection{Mechanism Ablations}
\label{sec:ablation}
\begin{table}[htbp]
\caption{ADMET-PMO ablations over five tasks and five runs. The full-system reference matches Table~\ref{tab:admet_pmo}. Target AUC Sum adds task means. AA-AUC and PHV average tasks and runs. Best/second-best means are bold/underlined.}
\label{tab:ablation}
\centering
\small
\begin{tabular}{lccc}
\toprule
Variant & Target AUC Sum & AA-AUC & PHV \\
\midrule
\ours{} (full system) & \textbf{3.6936} & \textbf{0.7076} & \textbf{0.8248} \\
\quad w/o Lagrangian & \underline{3.6535} & 0.4781 & 0.7815 \\
\quad w/o Chemical-CER & 3.4008 & 0.5459 & \underline{0.7904} \\
\quad w/o Curriculum & 3.1981 & \underline{0.6260} & 0.7295 \\
\quad w/o MedChem-RAG & 2.9977 & 0.4146 & 0.6593 \\
\quad w/o Reviewer agent & 3.3024 & 0.5566 & 0.7100 \\
\quad Single-agent baseline & 3.2698 & 0.3000 & 0.6546 \\
\bottomrule
\end{tabular}
\end{table}
The full system outperforms all six variants on all three aggregate metrics (Table~\ref{tab:ablation}).
Removing Lagrangian feedback reduces AA-AUC by 32.4\%, compared with a 1.1\% reduction in target-AUC sum. The loss is concentrated in feasible-solution progress.
Removing CER reduces target-AUC sum by 7.9\% and AA-AUC by 22.9\%. Removing Curriculum reduces them by 13.4\% and 11.5\%, and PHV by 11.6\%.
MedChem-RAG removal produces the largest target-AUC reduction among component removals (18.8\%) and lowers AA-AUC by 41.4\%.
Reviewer removal lowers AA-AUC by 21.3\%. The full system reaches 0.7076 AA-AUC versus 0.3000 for the single-agent configuration.
The distinct loss profiles reveal complementary contributions from constraint feedback, structural experience, scheduling, chemical priors, and review.

\subsection{Molecular Cases of Property-Violation Resolution}
\label{sec:cases}
Figure~\ref{fig:cases}(a--e) compares property profiles before full constraint activation (round~59) and at termination, using run~0 uniformly across tasks. Figure~\ref{fig:cases}(f) shows the corresponding GSK3$\beta$ structures.
\ours{} resolves property violations in all five illustrated trajectories. Every earlier state violates at least one of the final seven requirements, and every final state satisfies all seven.
Target scores increase for GSK3$\beta$ (0.20 to 0.92) and Sitagliptin (0.1443 to 0.4533), remain at 1.00 for Albuterol, and decrease for JNK3 (1.00 to 0.86) and DRD2 (1.0000 to 0.9864).
For GSK3$\beta$, the predicted hERG block score falls from 0.5754 to 0.3496 (bound 0.5), while Caco-2 increases from $-5.5022$ to $-4.4525$ (bound $-5.15$).
Appendix~\ref{app:cases} provides all structure pairs, raw properties, and the selection protocol.

\section{Discussion and Conclusion}
\ours{} achieves leading aggregate performance through adaptive external search state and frozen language models. E1 leads aggregate target search, and E2 outperforms all six online strategies on all five aggregate metrics. The Lagrangian controller gives constraint history an explicit role in acceptance and proposal feedback, while Chemical-CER and the curriculum supply relevant measured experience and coordinate constraint activation. Removing Lagrangian feedback lowers AA-AUC by 32.4\%. Every component removal lowers all three aggregate metrics, and all five molecular cases resolve their property violations. Together, these results establish adaptive constraint feedback and structured chemical experience as a powerful approach to molecular search with frozen models.

\paragraph{Limitations.}
Evaluation covers five tasks and predicted properties. Multiplier boundedness remains an assumption of the appendix analysis, and ablation uncertainty is unavailable. Appendix~\ref{app:limitations} summarizes the few task-level exceptions, validation needs, and computational-cost considerations.

\label{sec:main_end}
\raggedbottom
\subsection*{AI Use Statement}
Generative AI tools were used solely for language editing and polishing, including improvements to grammar, wording, clarity, and readability. All scientific content, methodology, experiments, analyses, results, and conclusions were developed by the authors. All AI-assisted edits were reviewed and verified by the authors, who take full responsibility for the final manuscript.

\subsection*{Ethics statement}
This study evaluates computational molecular optimization using predicted properties. It does not establish experimental drug safety or clinical efficacy.
Candidate structures and predicted ADMET measurements require independent validation before practical use.
We discuss potential misuse and the limits of surrogate evaluation in Appendix~\ref{app:impact}.

\subsection*{Reproducibility statement}
We will release the complete source code for running the \ours{} agent system upon acceptance.

\bibliographystyle{iclr2027_conference}
\bibliography{references}

%=============================================================================
\newpage
\appendix
\setcounter{figure}{0}
\renewcommand{\thefigure}{A\arabic{figure}}
\renewcommand{\theHfigure}{appendix.\arabic{figure}}
\setcounter{table}{0}
\renewcommand{\thetable}{A\arabic{table}}
\renewcommand{\theHtable}{appendix.\arabic{table}}
\setcounter{algorithm}{0}
\renewcommand{\thealgorithm}{A\arabic{algorithm}}
\providecommand{\theHalgorithm}{\arabic{algorithm}}
\renewcommand{\theHalgorithm}{appendix.\arabic{algorithm}}
\setcounter{equation}{0}
\renewcommand{\theequation}{A\arabic{equation}}
\renewcommand{\theHequation}{appendix.\arabic{equation}}
\setlength{\textfloatsep}{9pt plus 2pt minus 2pt}
\setlength{\floatsep}{8pt plus 2pt minus 2pt}
\setlength{\intextsep}{9pt plus 2pt minus 2pt}
\captionsetup[figure]{skip=5pt}
\captionsetup[table]{skip=5pt}

\FloatBarrier
\section*{Appendix}
The appendices follow the main paper's questions. Appendix~\ref{app:proof} derives multiplier accumulation and conditional feasibility. Appendix~\ref{app:statistics} reports run-level uncertainty and complete task results. Appendix~\ref{app:cases} documents the molecular cases, and Appendix~\ref{app:online} details the online adaptations, task-level margins, and changes from the original protocols. Appendices~\ref{app:scales}--\ref{app:compute} specify metrics, prompts, replay, and computational accounting. Appendix~\ref{app:limitations} collects the limitations and remaining validation questions.

\section{Cumulative Violation and Conditional Feasibility}
\label{app:proof}

\paragraph{Update sequence.}
Here $t$ indexes completed updates after the full constraint set is active. Empty or failed rounds are excluded, and $\boldsymbol\lambda^{(0)}\geq0$ retains earlier-phase history.
The state $\bar x^{(t)}$ is the accepted candidate or retained incumbent from Eq.~\ref{eq:accept}. Write $v_i^{(t)}=v_i(\bar x^{(t)})\geq0$ and $V^{(t)}=\sum_i v_i^{(t)}$.

\paragraph{Exact accumulation.}
Since $\rho^{(t)}>0$ and violations are nonnegative, Eq.~\ref{eq:dual_update} keeps every multiplier nonnegative. Any nonnegative projection is therefore inactive.
Telescoping the update for a single constraint gives
\[
\lambda_i^{(T)}-\lambda_i^{(0)}
=\sum_{t=0}^{T-1}\rho^{(t)}v_i^{(t)}.
\]
Summing over $i$ yields Eq.~\ref{eq:violation_accounting}. These identities hold for every finite update sequence and follow algebraically from the rule, without assumptions on convexity or proposal quality.

\begin{proposition}[Conditional vanishing violation]
\label{prop:conditional_feasibility}
Consider an infinite sequence of completed updates under Eq.~\ref{eq:dual_update} after full constraint activation.
Suppose $\rho^{(t)}\geq\rho_{\min}>0$ for all $t$ and, for each $i$, there is a finite constant $M_i$ such that $\lambda_i^{(t)}\leq M_i$ for all $t$.
Then $\sum_{t=0}^{\infty}V^{(t)}<\infty$ and $V^{(t)}\to0$.
\end{proposition}
\begin{proof}
For every finite $T$, the accumulation identity gives
\[
0\leq\sum_{t=0}^{T-1}\rho^{(t)}V^{(t)}
=\sum_{i=1}^{K}(\lambda_i^{(T)}-\lambda_i^{(0)})
\leq\sum_{i=1}^{K}(M_i-\lambda_i^{(0)}).
\]
The upper bound is independent of $T$. Nonnegativity and $\rho^{(t)}\geq\rho_{\min}$ therefore imply
\[
\sum_{t=0}^{\infty}V^{(t)}
\leq\frac{1}{\rho_{\min}}\sum_{i=1}^{K}(M_i-\lambda_i^{(0)})<\infty.
\]
A summable nonnegative sequence tends to zero.
\end{proof}

\paragraph{Interpretation of the assumptions.}
For $\gamma\geq1$ and $\rho_{\max}\geq\rho_0>0$, the penalty update and active-set resets preserve $\rho^{(t)}\geq\rho_0$.
The cap on $\rho$ does not bound the multipliers, whose increments continue to accumulate.
Uniform multiplier boundedness is an assumption of Proposition~\ref{prop:conditional_feasibility} and is not established for the implemented LLM search.
The proposition describes a consequence of controlled cumulative violation, not an unconditional convergence guarantee for the search.

\paragraph{Limit points.}
Under the proposition's assumptions, suppose a subsequence $\bar x^{(t_m)}$ converges to $x^\star\in\mathcal X$ in a fixed metric and each $v_i$ is continuous at $x^\star$.
Then $0\leq v_i(\bar x^{(t_m)})\leq V^{(t_m)}\to0$ implies $v_i(x^\star)=0$ for every $i$.
This establishes feasibility of such a limit point, without asserting that the molecular sequence has a limit point or converges.

%=============================================================================
\FloatBarrier
\section{Statistical Reporting and Detailed Results}
\label{app:statistics}
All five tasks and five runs are retained. For a fixed-task aggregate, first average the five task values within each run index and then estimate uncertainty from the five resulting observations. Individual molecules and the 25 task/run cells are not treated as independent repetitions.
For values $z_1,\ldots,z_5$, intervals are $\bar z\pm t_{0.975,4}s_z/\sqrt5$, with $t_{0.975,4}=2.776445$.
For E2--E1 comparisons, $z_j$ is the within-index difference. Index matching is not a claim of common LLM random draws.
Cross-method differences use $\sqrt{s_A^2/5+s_B^2/5}$ and the same conservative critical value.
These intervals assume approximately normal run-level variation and are unadjusted descriptive intervals, not simultaneous claims over all comparisons.
AUC sums are five times their corresponding fixed-task means, so their intervals scale by five and provide no additional independent evidence.

For aggregate E2--E1 AA-AUC, the difference is 0.1749 with matched interval $[0.0511,0.2987]$. The independent-group sensitivity check gives $[0.0195,0.3303]$.
E2--E1 aggregate target AUC, Top-1, Top-10, and PHV intervals cross zero, which does not establish equivalence.
Against original MOLLEO with post-hoc annotation, E2 has positive aggregate difference intervals for Top-1, target AUC, AA-AUC, and PHV. The final Top-10 interval crosses zero. Against all six online strategies, all 18 target-AUC, AA-AUC, and PHV difference intervals are positive (Figure~\ref{fig:online_summary}(d--f)). These are separate, unadjusted comparisons. The online AA-AUC differences are 0.2474 versus ExLLM FF, 0.3070 versus MOLLEO FF, and 0.3238 versus SEISMO FF.
\begin{figure}[htbp]
\centering\includegraphics[width=\linewidth]{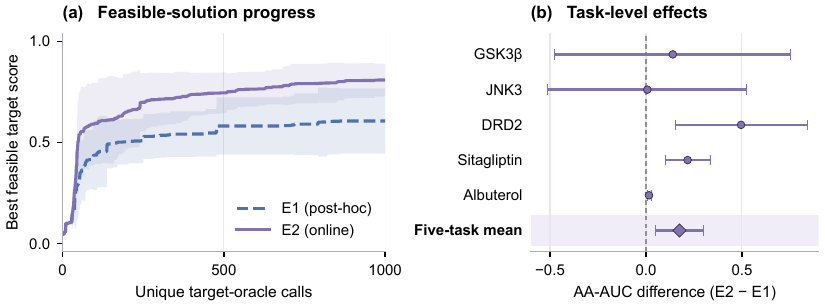}
\caption{Online control improves feasible-solution search. (a) Best feasible target score versus unique oracle calls, including initialization. Lines average five task-averaged runs, with pointwise 95\% intervals. (b) E2--E1 AA-AUC differences for every task and the fixed five-task mean, with run-index-matched 95\% intervals ($n=5$, unadjusted). Shared indices do not imply shared LLM draws. An independent-group aggregate check remains positive (Appendix~\ref{app:statistics}).}
\label{fig:search_evidence}
\end{figure}
\paragraph{Online and original-protocol comparisons.}
Figures~\ref{fig:supp_online_auc_top_10}--\ref{fig:supp_online_pareto_hypervolume} show every online strategy. Figures~\ref{fig:supp_auc_top_10}--\ref{fig:supp_pareto_hypervolume} retain the original baselines and PharmAgent E1/E2.
In each figure, panels (a--e) show GSK3$\beta$, JNK3, DRD2, Sitagliptin MPO, and Albuterol similarity, respectively, and panel (f) shows the fixed-task aggregate.
All five runs are shown. Table~\ref{tab:posthoc} retains the post-hoc aggregate comparison. Relative to original MOLLEO, E2 improves target-AUC sum by 29.3\%, AA-AUC by 70.7\%, and PHV by 17.6\%.

\clearpage
\begin{figure}[p]\centering
\includegraphics[width=\linewidth]{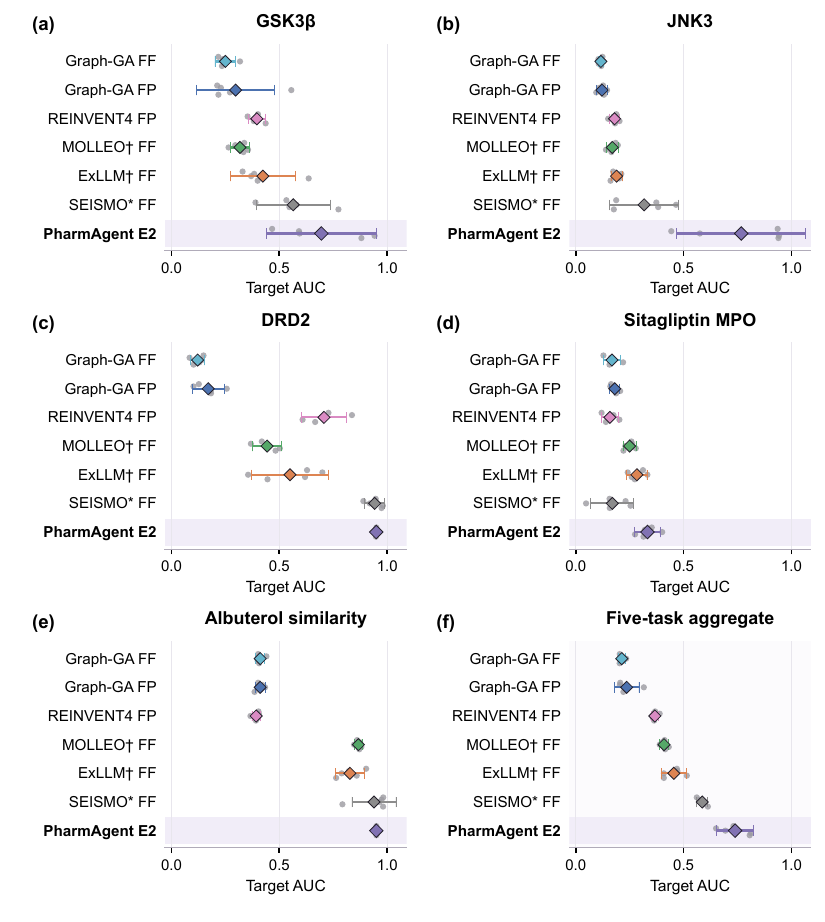}
\caption{Target AUC across all six online baseline strategies and PharmAgent E2. Diamonds denote means, grey dots the five runs, and bars unadjusted 95\% Student-$t$ intervals. Panels (a--e) show the five tasks. Panel (f) averages tasks within each run before estimating uncertainty. Intervals are not clipped to $[0,1]$. $\dagger$ denotes independent MiniMax-M3 reimplementations. *SEISMO follows its short-run protocol with terminal-score carry-forward. FF denotes feasibility-first and FP fixed-penalty search.}
\label{fig:supp_online_auc_top_10}
\end{figure}

\clearpage
\begin{figure}[p]\centering
\includegraphics[width=\linewidth]{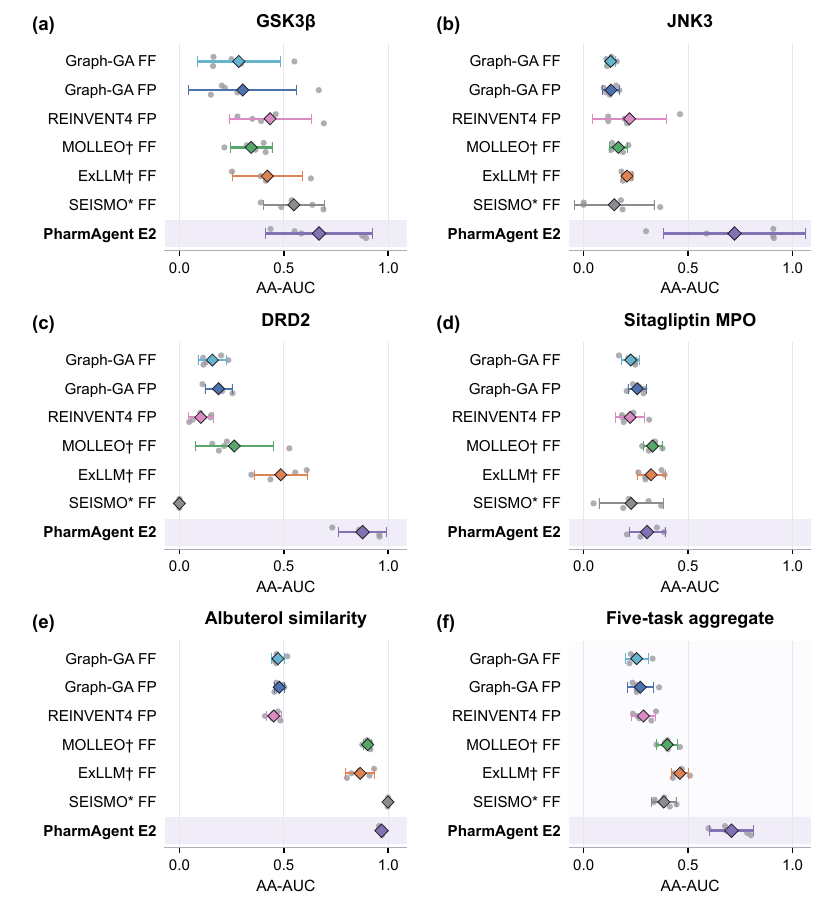}
\caption{AA-AUC across all six online baseline strategies and PharmAgent E2. Diamonds denote means, grey dots the five runs, and bars unadjusted 95\% Student-$t$ intervals. Panels (a--e) show the five tasks. Panel (f) averages tasks within each run before estimating uncertainty. Intervals are not clipped to $[0,1]$. $\dagger$ denotes independent MiniMax-M3 reimplementations. *SEISMO follows its short-run protocol with terminal-score carry-forward. FF denotes feasibility-first and FP fixed-penalty search.}
\label{fig:supp_online_aa_auc}
\end{figure}

\clearpage
\begin{figure}[p]\centering
\includegraphics[width=\linewidth]{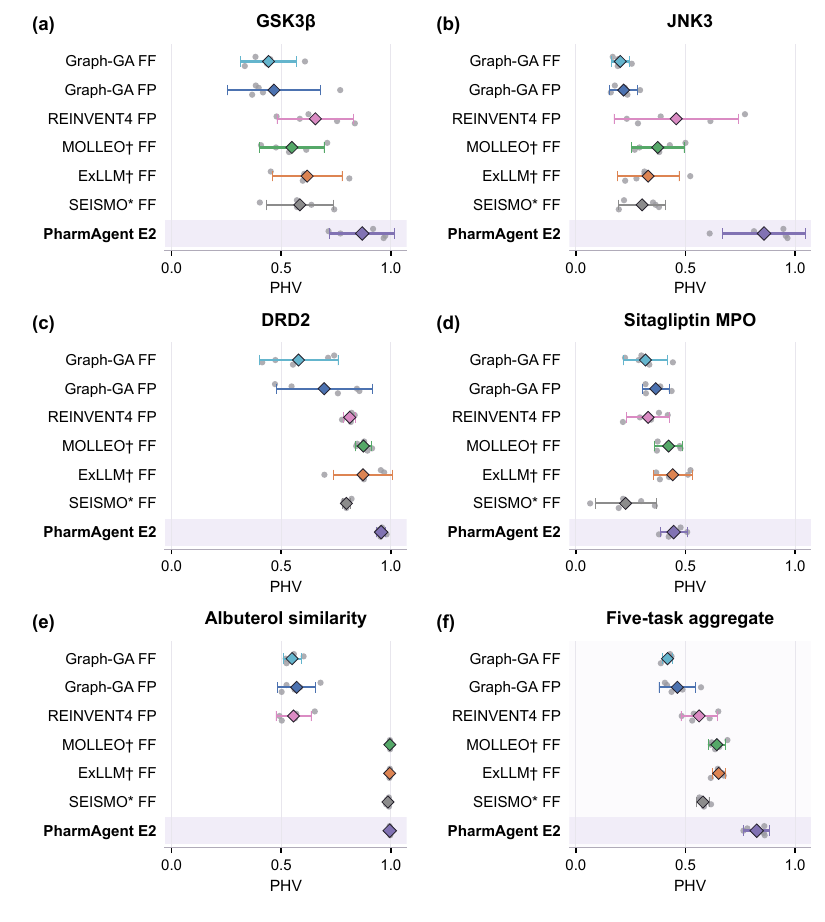}
\caption{Pareto hypervolume across all six online baseline strategies and PharmAgent E2. Diamonds denote means, grey dots the five runs, and bars unadjusted 95\% Student-$t$ intervals. Panels (a--e) show the five tasks. Panel (f) averages tasks within each run before estimating uncertainty. Intervals are not clipped to $[0,1]$. $\dagger$ denotes independent MiniMax-M3 reimplementations. *SEISMO follows its short-run protocol with terminal-score carry-forward. FF denotes feasibility-first and FP fixed-penalty search.}
\label{fig:supp_online_pareto_hypervolume}
\end{figure}

\clearpage
\begin{figure}[p]\centering
\includegraphics[width=\linewidth]{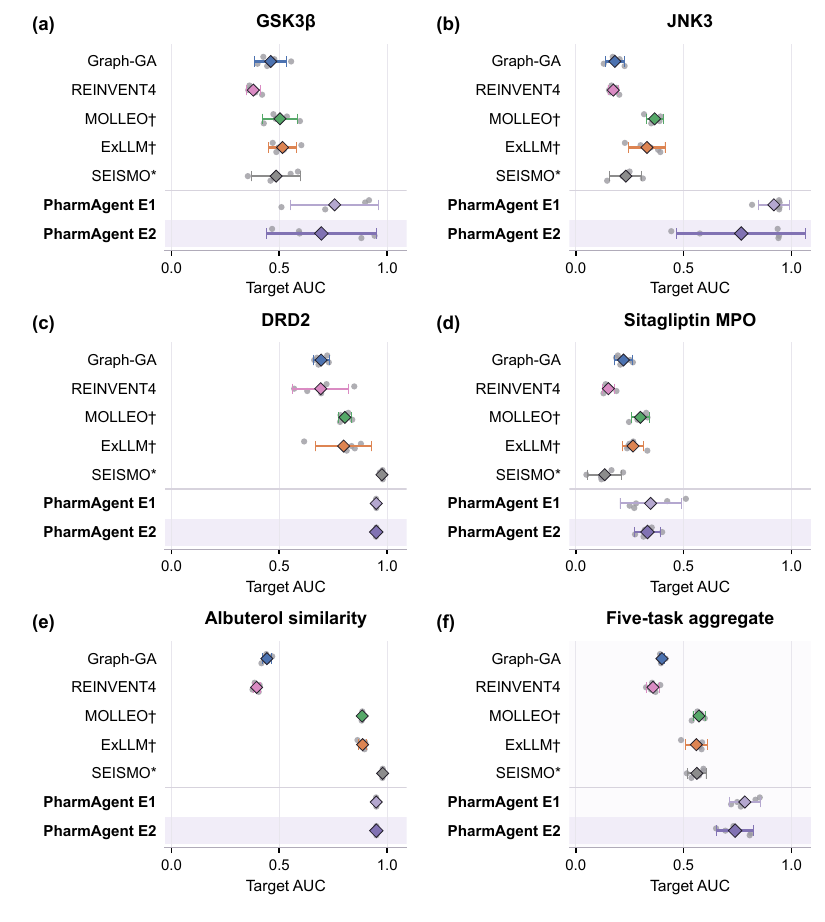}
\caption{Target AUC for the five original baselines and PharmAgent E1/E2. Baselines and E1 receive post-hoc property annotation, while E2 uses online constraints. Diamonds denote means, grey dots the five runs, and bars unadjusted 95\% Student-$t$ intervals. Panels (a--e) show the five tasks. Panel (f) averages tasks within each run before estimating uncertainty. Intervals are not clipped to $[0,1]$. $\dagger$ denotes independent MiniMax-M3 reimplementations. *SEISMO follows its short-run protocol with terminal-score carry-forward.}
\label{fig:supp_auc_top_10}
\end{figure}

\clearpage
\begin{figure}[p]\centering
\includegraphics[width=\linewidth]{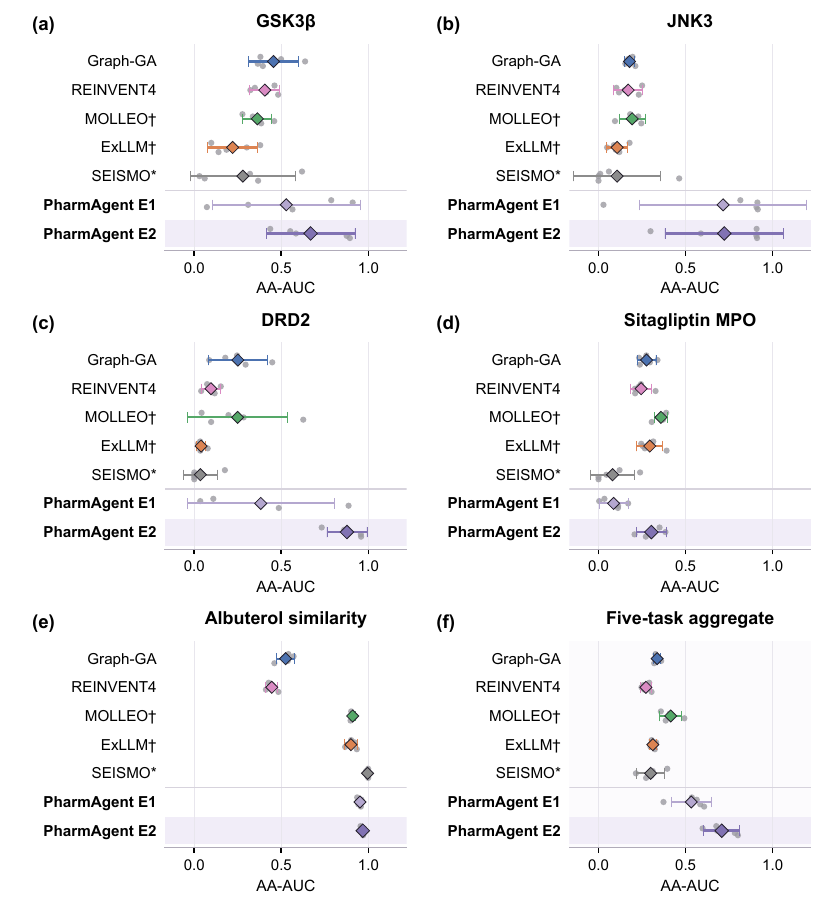}
\caption{AA-AUC for the five original baselines and PharmAgent E1/E2. Baselines and E1 receive post-hoc property annotation, while E2 uses online constraints. Diamonds denote means, grey dots the five runs, and bars unadjusted 95\% Student-$t$ intervals. Panels (a--e) show the five tasks. Panel (f) averages tasks within each run before estimating uncertainty. Intervals are not clipped to $[0,1]$. $\dagger$ denotes independent MiniMax-M3 reimplementations. *SEISMO follows its short-run protocol with terminal-score carry-forward.}
\label{fig:supp_aa_auc}
\end{figure}

\clearpage
\begin{figure}[p]\centering
\includegraphics[width=\linewidth]{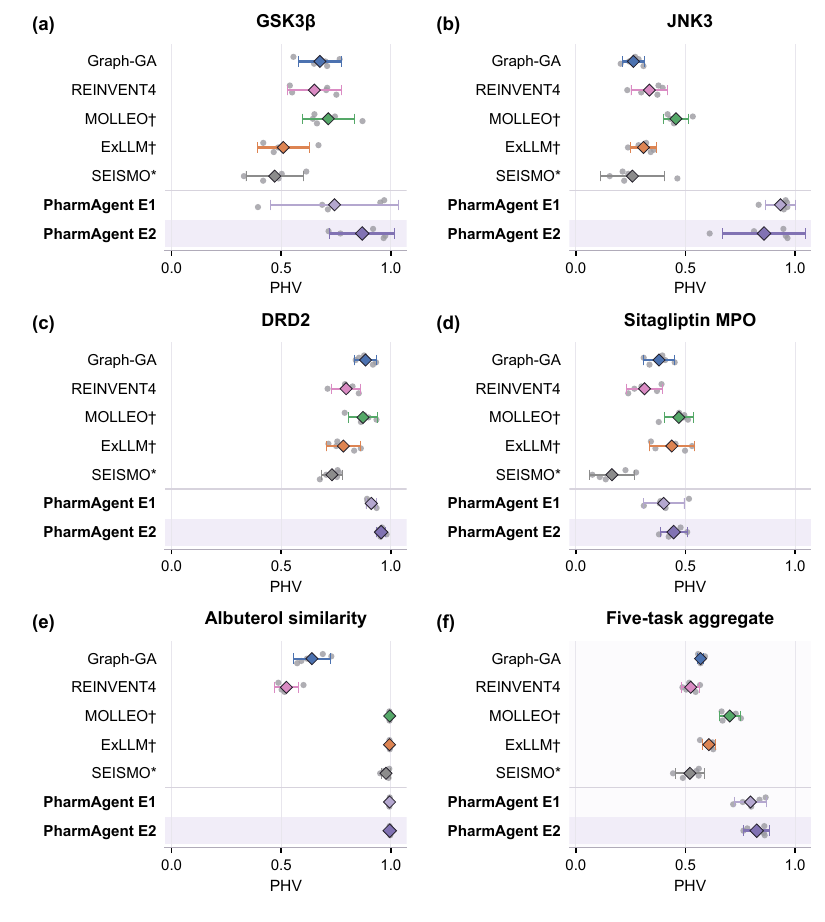}
\caption{Pareto hypervolume for the five original baselines and PharmAgent E1/E2. Baselines and E1 receive post-hoc property annotation, while E2 uses online constraints. Diamonds denote means, grey dots the five runs, and bars unadjusted 95\% Student-$t$ intervals. Panels (a--e) show the five tasks. Panel (f) averages tasks within each run before estimating uncertainty. Intervals are not clipped to $[0,1]$. $\dagger$ denotes independent MiniMax-M3 reimplementations. *SEISMO follows its short-run protocol with terminal-score carry-forward.}
\label{fig:supp_pareto_hypervolume}
\end{figure}
\clearpage

\FloatBarrier
\section{Molecular Case Selection and Property Measurements}
\label{app:cases}
The five cases in Figure~\ref{fig:cases} use a fixed rule that compares the accepted state at round~59 with the final accepted state in run~0 for every task.
Round~59 is the last round before all seven constraints are active under the prescribed schedule.
The same run and stage endpoints are used for every task, independently of the magnitude of the target gain.
A state can persist across many rounds. Its recorded evaluation call can therefore precede the displayed round.
All plotted properties reuse stored measurements, and RDKit renders the corresponding structures.

For each raw property $g_i$, the radar radius is the existing desirability
$d_i=\sigma(-5(\tilde g_i-\tilde t_i)/r_i)$, using the same directions, thresholds, and scales as Eq.~\ref{eq:admet_score}.
Thus a threshold maps to 0.5 and larger radii always indicate greater satisfaction.
No normalization is fitted to the displayed cases and polygon area is not used as a metric.
The sigmoid can saturate, so Table~\ref{tab:case_properties} retains raw values for inspection.
The final accepted state and the historical best feasible molecule are distinct search outputs. In JNK3 run~0, for example, the final accepted target score is 0.86, while the historical best feasible score is 0.95.
The full paired structures are shown in Figure~\ref{fig:all_case_structures}(a--e), in the same task order.
\begin{figure}[htbp]
\centering\includegraphics[width=\linewidth,trim=0 44bp 0 12bp,clip]{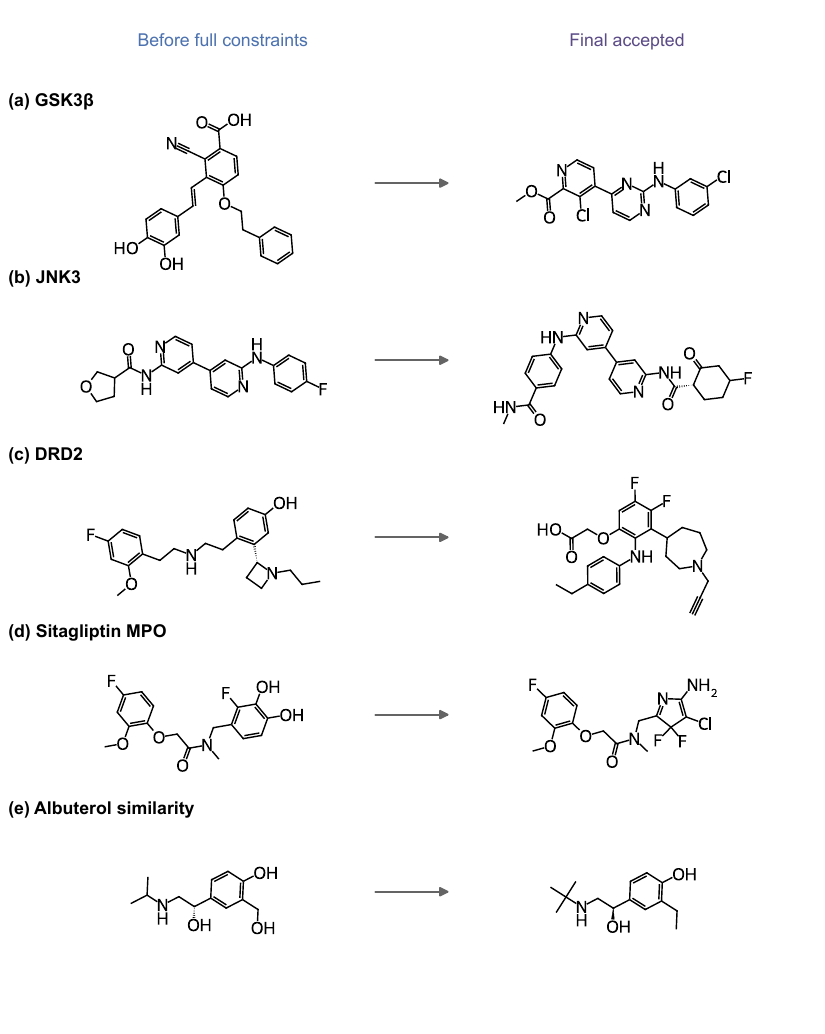}
\caption{All five recorded accepted-state structure pairs from E2 underlying Figure~\ref{fig:cases}. Panels (a--e) show GSK3$\beta$, JNK3, DRD2, Sitagliptin MPO, and Albuterol similarity. Every task uses run~0, comparing the state before full constraint activation (round~59) with the final accepted state. All structures use a common bond-length scale, atom-label size, and line width. Structures are drawn from original canonical SMILES. Target scores and violation counts appear in Figure~\ref{fig:cases}. Raw properties are reported in Table~\ref{tab:case_properties}.}
\label{fig:all_case_structures}
\end{figure}
\begin{table}[htbp]
\caption{Raw computed or predicted properties for the ten states in Figure~\ref{fig:cases}. All cases use run~0. Bold values violate the displayed threshold. CYP denotes CYP2D6 inhibition. MW is in Da and Caco-2 follows the predictor scale. R59 means the accepted state at round~59. Final means the final accepted state.}
\label{tab:case_properties}
\centering\footnotesize
\setlength{\tabcolsep}{3pt}
\begin{tabular}{llrrrrrrr}\toprule
Task & State & MW & logP & QED & SA & CYP & hERG & Caco-2 \\
Bound & & $\leq500$ & $\leq5$ & $\geq0.3$ & $\leq6$ & $\leq0.5$ & $\leq0.5$ & $\geq-5.15$ \\\midrule
GSK3$\beta$ & R59 & 401.418 & 4.4596 & 0.3982 & 2.4764 & 0.0553 & \textbf{0.5754} & \textbf{-5.5022} \\
GSK3$\beta$ & Final & 375.215 & 4.3756 & 0.6825 & 2.3569 & 0.4371 & 0.3496 & -4.4525 \\
\addlinespace[2pt]
JNK3 & R59 & 378.407 & 4.0013 & 0.7018 & 2.7675 & 0.0638 & \textbf{0.6968} & -4.9562 \\
JNK3 & Final & 461.497 & 3.8926 & 0.4802 & 3.4456 & 0.0391 & 0.4993 & -5.1188 \\
\addlinespace[2pt]
DRD2 & R59 & 386.511 & 4.0715 & 0.6044 & 2.9680 & \textbf{0.9720} & \textbf{0.9889} & \textbf{-5.2292} \\
DRD2 & Final & 442.506 & 4.9369 & 0.5675 & 3.1650 & 0.4963 & 0.4977 & -5.0343 \\
\addlinespace[2pt]
Sitagliptin & R59 & 353.321 & 2.4220 & 0.7800 & 2.2937 & 0.1481 & \textbf{0.6927} & -4.7337 \\
Sitagliptin & Final & 377.750 & 2.1280 & 0.8235 & 3.0148 & 0.0292 & 0.4788 & -4.5046 \\
\addlinespace[2pt]
Albuterol & R59 & 225.288 & 0.9159 & 0.6010 & 2.6694 & 0.1282 & 0.1020 & \textbf{-5.4044} \\
Albuterol & Final & 237.343 & 2.3761 & 0.7533 & 2.6696 & 0.1501 & 0.3435 & -4.9597 \\
\bottomrule\end{tabular}
\end{table}

\FloatBarrier
\section{Online Baseline Adaptations and Detailed Comparisons}
\label{app:online}
\subsection{Shared Comparison Design}
The six online strategies span all five baseline frameworks, with five tasks and five formal runs each (150 task/run cells). PharmAgent E2 contributes 25 reference cells. Protocols were fixed before observing the formal evaluation results. All seven constraints, their thresholds, and normalized violation scales match E2. Each unique molecule receives the same property prediction procedure. Graph-GA, REINVENT4, MOLLEO, and ExLLM reach 1,000 unique target evaluations per cell. SEISMO retains its original short-run protocol.

Figure~\ref{fig:online_summary}(a--c) shows aggregate performance with uncertainty for every online strategy. Figure~\ref{fig:online_summary}(d--f) shows the corresponding PharmAgent E2 advantage over each baseline. All 18 difference intervals are positive.
\begin{figure}[htbp]
\centering\includegraphics[width=\linewidth]{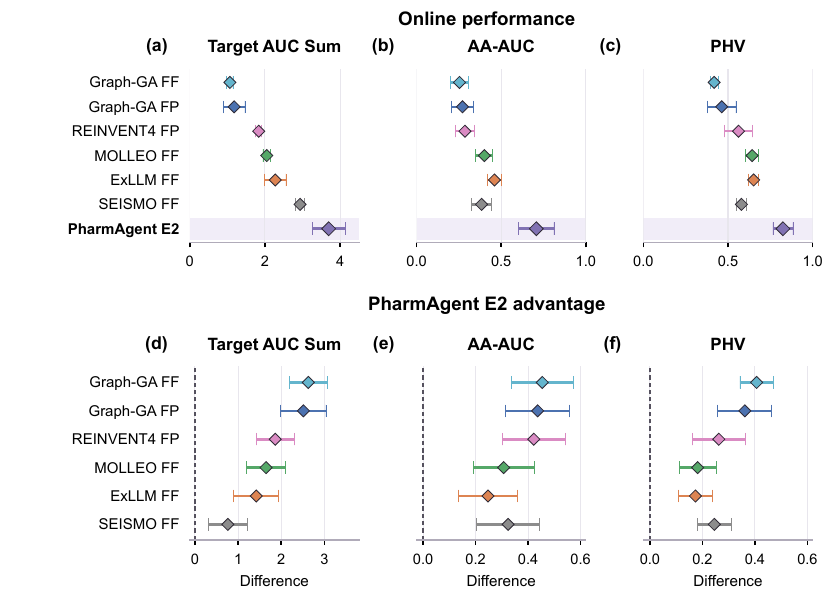}
\caption{Online performance and comparative gains across five tasks and five runs. Panels (a--c) show five-run means with unadjusted 95\% Student-$t$ intervals. Panels (d--f) show PharmAgent E2 minus each baseline with unadjusted 95\% independent-group intervals. Target AUC uses five-task sums, while AA-AUC and PHV use task means. Each interval uses five run-level aggregates per method. The zero line in (d--f) denotes equal aggregate performance. All six baseline strategies are retained. FF denotes feasibility-first and FP fixed-penalty search. SEISMO retains its short-run protocol and terminal-score carry-forward.}
\label{fig:online_summary}
\end{figure}

\subsection{Optimizer and Language-Model Adaptations}
\paragraph{Non-LLM adaptations.}
Graph-GA uses FF and FP, and REINVENT4 uses FP. All seven constraints enter rewards from the first evaluation.
FP fixes penalty weight $w$ within each run and uses reward $\sigma(4(f(x)-w\sum_i v_i(x)-0.5))$.
Graph-GA's FF reward is $0.5+0.5f(x)$ if all seven constraints pass, and $0.5\exp(-\sum_i v_i(x))$ otherwise.
Metrics use raw target and property records, separately from optimizer rewards.

For each optimizer, $w\in\{0.1,1,10\}$ is selected by mean AA-AUC over all five tasks and two separate development runs (indices 100/101, 500 calls each). Ties within $10^{-12}$ use PHV then smaller weight.
Graph-GA selects 10 and REINVENT4 selects 1. All development outcomes are retained and excluded from formal means.
Graph-GA uses population 120, offspring 70, mutation rate 0.067, and ZINC initialization.
REINVENT4 uses its original prior, batch size 64, $\sigma=128$, learning rate $10^{-4}$, and an identical Murcko-scaffold diversity filter.
All 75 non-LLM formal cells reach 1,000 unique evaluations. Invalid proposals and cache hits consume no new target calls.
The 60 development runs consume 30,000 additional target evaluations. The 75 formal runs consume 75,000.
These controls use no LLM calls, while target and ADMET measurements run on the shared server.
\begin{table}[htbp]\centering\small
\caption{Independent development runs for selecting fixed-penalty weights. Each row averages five tasks and two runs. Formal outcomes do not enter selection.}
\begin{tabular}{lcccc}\toprule Optimizer & Weight & AA-AUC & PHV & Selected \\\midrule
Graph-GA & 0.1 & 0.206613 & 0.349602 & No \\
Graph-GA & 1 & 0.209172 & 0.294164 & No \\
Graph-GA & 10 & 0.209456 & 0.290444 & Yes \\
REINVENT4 & 0.1 & 0.257339 & 0.489099 & No \\
REINVENT4 & 1 & 0.286928 & 0.542494 & Yes \\
REINVENT4 & 10 & 0.262841 & 0.506246 & No \\
\bottomrule\end{tabular}\end{table}

\paragraph{LLM feasibility-first adaptations.}
MOLLEO, ExLLM, and SEISMO retain MiniMax-M3, their task descriptions, and response formats. Prompts add the common constraint objective and relevant parents' or historical molecules' property values, thresholds, normalized violations, and feasibility. These adaptations use no PharmAgent trajectories, Chemical-CER, adaptive multipliers, or curriculum.
For MOLLEO and ExLLM population selection, feasible candidates precede infeasible candidates. Feasible candidates are ordered by decreasing raw target score and infeasible candidates by increasing total normalized positive violation. Canonical SMILES breaks ties.
MOLLEO retains roulette parent sampling using a positive feasibility reward and uses this lexicographic ordering for survivors. Its population size is 120, offspring count 70, and mutation rate 0.067. ExLLM retains uniform parent sampling, a population of 50, two proposals per call, and experience injection probability 0.5. Its high/low evidence groups use the same feasibility-first order. Crossover and mutation probabilities are 0.8 and 0.2.
SEISMO receives the same feasibility-first priority through its conversational objective and measured constraint feedback, retaining its original proposal and decision process.
The hosted model uses temperature 0.7 and at most 2,048 completion tokens per request. Run indices label repeated runs rather than deterministic hosted-model seeds. Functional pilots are excluded from all reported means.

\subsection{Budget and Measurement Accounting}
SEISMO uses 51 generation transitions, the last an unscored summary, and at most 50 target-evaluation events. Repeated target evaluations count as events, while property predictions are cached by canonical molecule. After termination, the final Top-10 and historical best feasible scores are carried forward to the common 1,000-call horizon without further evaluations or molecules. Table~\ref{tab:online_budget} reports unique evaluation counts and evaluation events.
\begin{table}[htbp]\centering\small
\caption{Target-evaluation counts across 25 task/run cells per strategy. Ranges are per cell. FF denotes feasibility-first and FP fixed-penalty search. Repeated SEISMO evaluations add events without adding unique molecules. All AUCs use the same 1,000-call horizon.}
\label{tab:online_budget}
\begin{tabular}{lrrr}\toprule
Method & Unique evaluations & Evaluation events & Total unique \\\midrule
Graph-GA (FF) & 1,000 & 1,000 & 25,000 \\
Graph-GA (FP) & 1,000 & 1,000 & 25,000 \\
REINVENT4 (FP) & 1,000 & 1,000 & 25,000 \\
MOLLEO (FF) & 1,000 & 1,000 & 25,000 \\
ExLLM (FF) & 1,000 & 1,000 & 25,000 \\
SEISMO (FF) & 6--47 & 37--50 & 955 \\
\ours{} (E2) & 1,000 & 1,000 & 25,000 \\
\bottomrule\end{tabular}\end{table}

\FloatBarrier
\subsection{Task-Level Margins Against Every Online Strategy}
Figure~\ref{fig:online_task_margins}(a) shows positive target-AUC differences in all 30 method--task comparisons. The smallest is 0.0073 against SEISMO on DRD2. Figure~\ref{fig:online_task_margins}(b) shows 27 positive AA-AUC differences. The exceptions are MOLLEO and ExLLM on Sitagliptin and SEISMO on Albuterol. Figure~\ref{fig:online_task_margins}(c) shows 29 positive PHV differences, with a 0.0004 reversal against MOLLEO on Albuterol. Each cell subtracts two five-run task means. The counts summarize these task-level mean comparisons. Figures~\ref{fig:supp_online_auc_top_10}--\ref{fig:supp_online_pareto_hypervolume} provide all runs and score intervals.

\begin{figure}[H]
\centering\includegraphics[width=\linewidth]{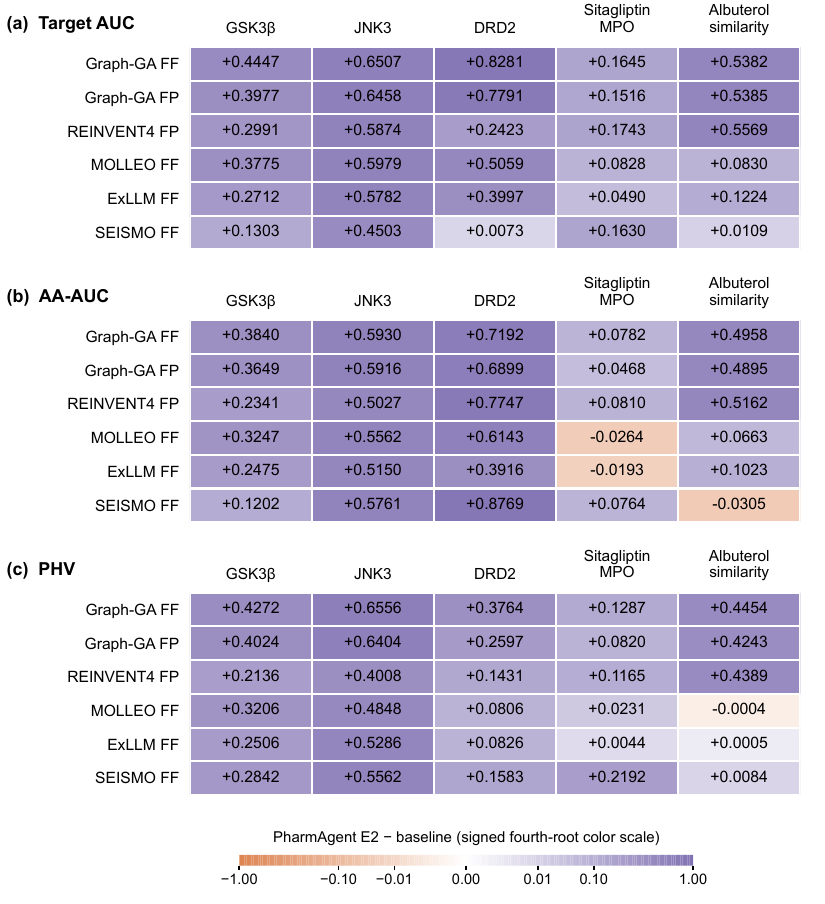}
\caption{PharmAgent E2 minus every online baseline on each task. Panels (a--c) show target AUC, AA-AUC, and PHV. Each cell is a difference between five-run means, calculated before rounding. Positive values favor PharmAgent. A shared signed color scale uses purple for positive differences and orange for negative differences. Cell fills and the color bar follow the same signed fourth-root scale, which makes small signed differences visible. Printed values are rounded to four decimals. All 90 comparisons are retained. Target AUC is reported per task. Colors encode mean differences, with uncertainty reported in Figures~\ref{fig:supp_online_auc_top_10}--\ref{fig:supp_online_pareto_hypervolume}. FF denotes feasibility-first and FP fixed-penalty search.}
\label{fig:online_task_margins}
\end{figure}

\clearpage
\subsection{Changes from Original to Online Search}
Figure~\ref{fig:online_adaptation_changes} complements the direct comparison by showing how each framework's outcomes change after online constraint handling is introduced. Each baseline adaptation is compared with that framework's original run configuration. PharmAgent E2 is compared with E1. The original trajectories receive the same post-hoc property evaluation, enabling comparison of AA-AUC and PHV on a common scale.

The three panels separate changes in target progress (Figure~\ref{fig:online_adaptation_changes}(a)), feasible-solution progress (Figure~\ref{fig:online_adaptation_changes}(b)), and two-objective coverage (Figure~\ref{fig:online_adaptation_changes}(c)). ExLLM's AA-AUC increases from 0.3121 to 0.4602, with a difference interval of $[0.1033,0.1929]$, while its target-AUC sum decreases by 0.5218. PharmAgent's AA-AUC increases from 0.5328 to 0.7076. Its independent-group difference interval is $[0.0195,0.3303]$, consistent with the run-index-matched comparison in the main paper. Its target-AUC-sum change of $-0.2272$ has interval $[-0.7880,0.3336]$.

The Graph-GA adaptations lower the three aggregate means under the tested configurations. REINVENT4 and SEISMO increase all three means, although their difference intervals cross zero. MOLLEO's changes in AA-AUC and PHV also have intervals spanning zero. These framework-specific responses show that access to constraint measurements and effective use of those measurements are distinct experimental questions. The direct E2 comparison in Table~\ref{tab:admet_pmo} evaluates the resulting online search methods.

\begin{figure}[htbp]
\centering\includegraphics[width=\linewidth]{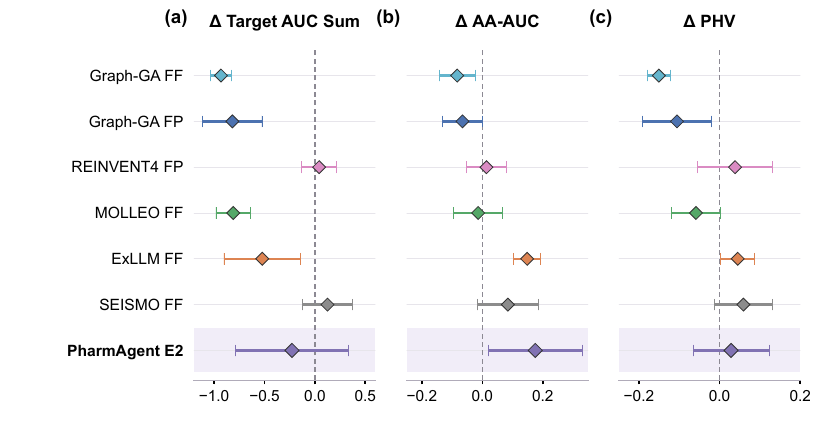}
\caption{Online minus original aggregate performance within each framework. Panels (a--c) show changes in target-AUC sum, AA-AUC, and PHV. Diamonds denote mean differences and bars unadjusted 95\% independent-group intervals over five run-level task aggregates per configuration. Positive changes favor online search. Graph-GA's original reference is shared by its two adaptations. PharmAgent compares E2 with E1, using the independent-group sensitivity analysis. Each contrast compares complete original and online configurations. FF denotes feasibility-first and FP fixed-penalty search. SEISMO retains its short-run protocol in both conditions.}
\label{fig:online_adaptation_changes}
\end{figure}

All five tasks and all five runs are retained for every contrast. Target-AUC changes use task sums, whereas AA-AUC and PHV use task means. The uncertainty calculation follows Appendix~\ref{app:statistics} and does not assume shared hosted-model random draws. The 21 displayed intervals are individual descriptive intervals without multiplicity adjustment. Both Graph-GA contrasts share an original reference, so they are not independent replications.

\FloatBarrier
\section{ADMET-PMO Protocol Details}
\label{app:scales}

\begin{table}[htbp]
\caption{ADMET constraints in the ADMET-PMO benchmark.  The same seven constraints are used for all five tasks evaluated in this paper.}
\label{tab:constraints}
\centering
\small
\begin{tabular}{llccl}
\toprule
Constraint & Property & Type & Threshold & Tool \\
\midrule
Lipinski MW & Molecular weight & $\leq$ & 500 Da & RDKit \\
Lipinski logP & Lipophilicity & $\leq$ & 5.0 & RDKit \\
QED & Drug-likeness & $\geq$ & 0.3 & RDKit \\
SA Score & Synth.\ accessibility & $\leq$ & 6.0 & RDKit \\
CYP2D6 inhibition & Metabolic safety & $\leq$ & 0.5 & ADMET-AI \\
hERG blocking & Cardiac safety & $\leq$ & 0.5 probability & ADMET-AI \\
Caco-2 & Permeability & $\geq$ & $-5.15$ & ADMET-AI \\
\bottomrule
\end{tabular}
\end{table}
QED and SA follow \citet{bickerton2012qed} and \citet{ertl2009sascore}. Learned endpoints use ADMET-AI~\citep{swanson2024admetai}.

\subsection{Consistent Evaluation Protocol}

Online ADMET-PMO measurements and post-hoc baseline annotations follow the same evaluation procedure.
Each molecule is represented by its RDKit-canonical SMILES~\citep{weininger1988smiles,rdkit} and evaluated individually.
Predictions are stored in a cache of up to 10,000 molecules, with the least recently used entries removed when capacity is reached.
This procedure gives consistent predictions across input batch sizes, molecule orderings, and prior cache contents. Cached results are copied before reuse.
The evaluation protocol fixes backend versions, endpoint mappings, constraints, normalization scales, and failure handling.
Online measurements and post-hoc baseline annotations are checked for consistency with this protocol before results are aggregated.

\subsection{Property-Specific Normalization Scales}

Constraint violations are normalized by property-specific scales so that the sigmoid transition zone is comparable across properties with vastly different natural ranges:

\begin{table}[htbp]
\caption{Property-specific normalization scales used in the ADMET score and Lagrangian penalty terms.}
\label{tab:scales}
\centering
\small
\begin{tabular}{lccl}
\toprule
Property & Scale $r_i$ & Steepness $s_i$ & Justification \\
\midrule
MW & 150\,Da & 5.0 & Std.\ dev.\ in drug-like space \\
logP & 2.0 & 5.0 & Std.\ dev.\ in drug-like space \\
QED & 0.2 & 5.0 & Range $[0,1]$, std $\approx 0.2$ \\
SA Score & 2.0 & 5.0 & Range $[1,10]$, std $\approx 2$ \\
CYP2D6 inh.\ & 0.3 & 5.0 & Probability scale $[0,1]$ \\
hERG blocking & 0.3 & 5.0 & Probability scale $[0,1]$ \\
Caco-2 perm.\ & 1.0 & 5.0 & Log permeability units \\
\bottomrule
\end{tabular}
\end{table}

\subsection{Continuous ADMET Score}
For full-constraint evaluation, the ADMET Analyst computes a continuous score by applying a sigmoid to each normalized distance from its threshold:
\begin{equation}
\text{ADMET}_{\text{score}}(x) = \frac{1}{\sum_i w_i} \sum_{i=1}^{K} w_i \cdot \sigma\!\!\left(\frac{s_i \cdot (\tilde t_i - \tilde g_i(x))}{r_i}\right),
\label{eq:admet_score}
\end{equation}
where $\sigma(\cdot)$ is the logistic sigmoid, $s_i$ is the steepness, $r_i$ is a property-specific normalization scale (Table~\ref{tab:scales}), and $w_i$ are configurable per-constraint weights.
The direction-normalized margin treats upper- and lower-bound requirements consistently.
This continuous score supports experience annotation and Pareto evaluation. Hard feasibility requires all seven inequalities, and candidate acceptance uses~\eqref{eq:lagrangian}.
The proposal model does not use gradients of the sigmoid score.
Prompt exposure follows the active constraints. Phase~1 and Standard PMO use target-only Generator and Reviewer templates.

\subsection{Formal Metric Definitions}

Let $\{(x_j, s_j)\}_{j=1}^N$ denote scored molecules, sorted by $s_j$ descending.
Let $\text{pass}(x) = \mathbb{1}[\forall i\colon \tilde g_i(x)\leq\tilde t_i]$.

\textbf{Pass rate among the top $k$ molecules (PR@$k$).}
$
\text{PR@}k = \frac{1}{k}\sum_{j=1}^{k} \text{pass}(x_j).
$

\textbf{AA-AUC.}
Let $s^*_{\text{pass}}(n)$ be the best score among passing molecules after $n$ oracle calls.
$
\text{AA-AUC} = \frac{1}{B}\sum_{n=1}^{B} s^*_{\text{pass}}(n).
$

\textbf{PHV.}
Using reference point $\mathbf{r}=(0,0)$:
$
\text{PHV} = \text{vol}\big(\bigcup_{j\in\mathcal{P}} [\mathbf{r},(s_j,a_j)]\big),
$
where $\mathcal{P}$ is the Pareto front of $(s_j, a_j)$ pairs and $a_j$ is the ADMET score of molecule~$j$.

%=============================================================================
\section{Prompt Templates}
\label{app:prompts}

\subsection{Generator Prompts}
The prompts request structured JavaScript Object Notation (JSON) responses.

\begin{quote}\small\ttfamily\raggedright
You are PharmAgent-Generator, an expert medicinal chemist AI.  You optimize molecules by proposing specific structural modifications with detailed pharmaceutical reasoning.

Rules: 1.\@ Output ONLY valid JSON.  2.\@ Every proposed SMILES must be chemically valid.  3.\@ Each modification must have explicit medicinal chemistry rationale.  4.\@ Consider only the explicitly active declared constraints.  5.\@ Output a JSON object with key ``candidates'' containing an array.
\end{quote}

\paragraph{Target-only and constrained branches.}
The quoted system prompt is used in constrained phases.
Standard PMO and curriculum Phase~1 instead use an explicit target-only system prompt that instructs the Generator to optimize only the stated target and omits all ADMET fields.
Both user-prompt branches receive the current SMILES and target score, ChEMBL matched-pair context from MedChem-RAG, previous Reviewer feedback, structured round-memory records, and the top-5 Chemical-CER experiences.
Only the constrained branch adds the current active-constraint ADMET summary and two explicitly separated blocks containing current measurements (value, threshold, direction, signed margin, and normalized violation) and historical multiplier pressure~$\lambda_i$.
The multiplier block records optimization history, while current measurements determine pass/fail status.
Pan-assay interference compounds (PAINS), Brenk, and Veber alerts appear in neither Generator branch and are never optimization guidance. If evaluated, they are post-hoc analyses outside the formal online state.
Every Generator request asks for at most five candidates.

\paragraph{Additional proposal instructions.}
When a restart is scheduled, the base prompt receives one proposal-strategy suffix (scaffold-hop, ring/linker re-architecture, or bioisosteric replacement). The suffix explicitly preserves the incumbent, supplied constraints, RAG evidence, and run history.
If and only if the primary response yields no valid molecules absent from the evaluated set, a second and final Generator request appends standardized rejection or duplication reasons and candidate indices and up to 20 previously evaluated canonical SMILES to avoid ordered by Morgan-Tanimoto similarity to the incumbent.
The second request preserves the base objective, active constraints, Verifier, and five-candidate cap.
Cross-round and cross-agent reuse is restricted to structured fields comprising canonical SMILES, finite scores and deltas, active-constraint facts, validated modification types, and normalized Reviewer priorities. Raw model prose is not forwarded as replay context. RAG exposes canonical structure pairs and finite target-score deltas. CER exposes canonical pairs and finite numerical deltas.

\paragraph{Phase-3 context-only repair anchors.}
After five consecutive non-improving rounds within the phase in Phase~3, each Generator request already permitted in that activated round receives the current archive of at most two infeasible Pareto anchors. This applies to the primary request and, when no valid novel candidate is produced, the single repair request.
The same archive is supplied to the Reviewer when the round produces scored candidates.
The highest-target and lowest-total-normalized-violation representatives are selected among observations that strictly outperform the best feasible target score.
Each anchor is accompanied by its SMILES and per-constraint value, threshold, direction, signed margin, and normalized violation.
The Generator is instructed to edit only the incumbent. The Verifier does not enforce structural similarity between the incumbent and a proposal.
The Reviewer receives the accepted state (the improving candidate or retained incumbent) as the incumbent for the next round.
Anchors provide measured examples for proposal revision. Acceptance, multiplier updates, and Chemical-CER recording follow their existing rules. CER stores the transition to the best evaluated candidate of the round rather than to the accepted iterate.
Thresholds, metrics, the Verifier, transition rules, candidate and request caps, and the oracle budget remain fixed. Through their influence on proposals, anchors can affect accepted states, multiplier updates, replay records, and the sequence of model requests, network attempts, and oracle evaluations.
They add bounded input tokens to those existing prompts.

\subsection{Reviewer Prompts}

\begin{quote}\small\ttfamily\raggedright
You are PharmAgent-Reviewer, a senior medicinal chemistry advisor.  You analyze optimization results and provide strategic guidance.

Rules: 1.\@ Output ONLY valid JSON.  2.\@ Be specific in suggestions: name exact structural modifications.  3.\@ Consider both activity and ADMET properties.  4.\@ Recommend termination if optimization has stagnated or goals are met.
\end{quote}

The quoted prompt is used only in constrained phases.
The target-only branch omits ADMET fields and directs the Reviewer not to use secondary properties.
For constrained rounds, every displayed candidate is paired with that candidate's own per-constraint value, threshold, direction, signed margin, and normalized violation. The incumbent's values are never substituted for a candidate's values, and historical $\lambda_i$ pressure remains a separate block.
When anchors are active, the Reviewer receives the same selected anchors as the Generator and uses the accepted state as the incumbent for the next round.
In formal runs with the 1,000-call budget, early termination is disabled, so the Reviewer is called at most once after a round scores at least one novel candidate and its termination recommendation cannot truncate the oracle budget.

%=============================================================================
\section{Chemical-CER Algorithm Details}
\label{app:cer}

\paragraph{Scores recorded with each transition.}
For a record collected in round $\tau$, let $J_\tau$ denote that round's scalar evaluation rule, with all its parameters fixed.
The record stores
\[
\Delta_{\mathrm{target}}=f(x_{\mathrm{tgt}})-f(x_{\mathrm{src}}),\qquad
\Delta_{\mathrm{combined}}=J_\tau(x_{\mathrm{tgt}})-J_\tau(x_{\mathrm{src}}).
\]
For the constrained full system, $J_\tau=\mathcal L_{\mathcal C_\tau}(\cdot,\boldsymbol\lambda^{(\tau)},\rho^{(\tau)})$. The target-only objective gives $J_\tau=f$.
The stored ADMET change describes the annotation used at collection time.
An entry remains an evaluated transition even when the candidate was rejected. Eligibility for later retrieval is a separate operation.

\paragraph{Eligibility before similarity.}
The retrieval interface distinguishes positive target-score change from positive collection-objective change.
Let $h_t(e)$ be the retrieval eligibility rule, determined by the improvement criterion and the permitted origin phases. Then $\mathcal D_t^{\mathrm{eligible}}=\{e\in\mathcal D:h_t(e)=1\}$.
Replay uses target-only information in Phases 1--2 and combined-objective information for all seven constraints in Phase~3.
Phase~1 exposes records originating before round~20 and Phase~2 before round~60. Phase~3 has no origin-round cutoff.
Retrieval uses completed records from earlier tasks and rounds of the current task. It excludes attempts from the ongoing round and incomplete records from failed rounds.
Historical task text and modification prose are not replayed.
The predicate is applied to recorded outcomes before source-structure distances are ranked.
Historical scalar changes retain their collection-time objective and multiplier values during retrieval.

\begin{algorithm}[htbp]
\caption{Chemical-CER: structure-aware experience retrieval}
\begin{algorithmic}[1]
\REQUIRE Experience database $\mathcal{D}=\{e_1,\ldots,e_N\}$, query molecule~$x$, integer~$k$, eligibility predicate $h_t$
\ENSURE Retrieved experiences $\mathcal{E}_k$
\STATE $\mathbf{fp}_x \leftarrow \textsc{MorganFP}(x,\,\text{radius}=2,\,\text{bits}=2048)$
\STATE $d_j \leftarrow \|\mathbf{fp}_x - \mathbf{fp}_{\text{src}}(e_j)\|_2$ for all $j=1,\ldots,N$ \hfill\COMMENT{Vectorized}
\STATE $\mathcal{S} \leftarrow \{e_j\in\mathcal{D} : h_t(e_j)=1\}$ \hfill\COMMENT{Apply configured eligibility}
\STATE $\mathcal{E}_k \leftarrow$ first $\min(k,|\mathcal S|)$ elements of $\mathcal{S}$ after stable sorting by ascending $d_j$
\STATE Format canonical structure pairs and finite numeric deltas for proposal conditioning
\RETURN $\mathcal{E}_k$
\end{algorithmic}
\end{algorithm}

The database stores fingerprints as a flat NumPy array and computes all distances in a single vectorized operation, avoiding the need for an external index (such as FAISS) at the typical scale of $<$1{,}000 entries.
The database persists across tasks through an optional serialization mechanism.
When the number of stored experiences exceeds the configured maximum (1{,}000 by default), the oldest entries are evicted in first-in, first-out (FIFO) order.

%=============================================================================
\section{Computational Resources and Reproducibility}
\label{app:compute}
\paragraph{Model and scientific evaluation.}
Generator and Reviewer access the provider-hosted MiniMax-M3 model through MiniMax's OpenAI-compatible application programming interface (API) in China~\citep{minimax2026m3api}.
Provider-side model versioning and sampling are addressed in Appendix~\ref{app:limitations}.
Run indices 0--4 identify repetitions and seed local random operations where applicable.
ADMET measurements use ADMET-AI~v1.4.0 with RDKit~\citep{swanson2024admetai,rdkit}.
Each uncached canonical molecule is evaluated individually and reused through the property cache.

\paragraph{Requests, rounds, and recovery.}
The formal target budget is 1,000 unique calls per task/run, including the 26-member seed pool.
Each primary Generator request proposes at most five candidates. A response with no valid novel molecules can trigger one final repair request with the same cap.
A scored round invokes at most one Reviewer.
Both E1 and E2 use a cap of 60,100 Hypertext Transfer Protocol (HTTP) attempts per task/run.
Observed round counts range from 458 to 3,592 for E1 and 538 to 6,616 for E2.
A round may contain no novel molecule, so rounds, HTTP attempts, and unique oracle evaluations are different quantities.
All E1/E2 cells reach 1,000 unique target calls. No performance-based early stopping is used.
The reported experiments use the request and round limits specified here.

Bounded novelty repair and proposal-context restarts preserve the incumbent and target budget.
After five consecutive non-improving rounds within the phase in Phase~3, up to two previously scored near-feasible anchors may enter existing prompts.
They supply context without directly replacing the incumbent or changing the acceptance rule.
The archive is evaluated separately from Chemical-CER, whose record remains the incoming-incumbent-to-round-best transition.
Stage changes occur only at rounds 20 and 60, and the penalty reset retains multipliers.
Appendix~\ref{app:prompts} describes the structured fields passed through recovery and feedback.

\paragraph{Recorded cost.}
Table~\ref{tab:cost} reports original provider usage across five tasks and five runs per arm.
Token totals measure recorded model usage. E2 uses 2.408 times E1's recorded total tokens.
Post-hoc E1 annotations reuse the 25,000 molecules from its existing generation runs.
\begin{table}[htbp]\centering\small
\caption{Recorded E1/E2 compute over 25 runs per arm. Oracle calls count unique target evaluations. HTTP attempts include retries. Round counts include rounds without new molecules.}
\label{tab:cost}
\begin{tabular}{lrr}\toprule Quantity & E1 & E2 \\\midrule
Unique target evaluations & 25,000 & 25,000 \\
Recorded optimization rounds & 38,013 & 49,854 \\
Successful HTTP calls & 83,653 & 108,055 \\
HTTP attempts & 83,751 & 108,168 \\
Input tokens & 119,050,374 & 301,475,520 \\
Output tokens & 55,388,732 & 118,552,476 \\
Total tokens & 174,439,106 & 420,027,996 \\
\bottomrule\end{tabular}\end{table}

The non-LLM online controls use 75,000 formal and 30,000 development target evaluations, with no LLM calls. MOLLEO FF and ExLLM FF contribute 25,000 unique formal target evaluations each. SEISMO FF evaluation counts appear in Table~\ref{tab:online_budget}. Oracle evaluation totals and LLM token/time costs are accounted for separately.
Graph-GA and REINVENT4 use the same server-side target/ADMET measurement environment. REINVENT4 optimization can execute locally through that measurement service.
Both optimizers therefore obtain measurements from the same evaluation service.
Missing baseline time/token measurements are left unreported. Appendix~\ref{app:limitations} discusses computational-cost comparability.

\paragraph{Finite controller observations.}
All 25 E2 runs terminate with accepted molecules passing the seven final constraints.
During Phase~3, the mean run-level fraction of feasible accepted rounds is 95.62\%. This includes repeated incumbents during rounds without new molecules.
It differs from the 29.13\% feasible fraction among evaluated molecules.
The largest observed multiplier is 1,815.5313.
Both E1 and E2 find an initially feasible seed at call~1. AA-AUC and best-feasible-score curves therefore track improvement beyond initial feasibility. These curves report finite-budget performance.
Trajectory analyses include every task and run, distinguishing violations of active constraints from violations under the full set of seven constraints.

\section{Hyperparameter Configuration}
\label{app:sensitivity}

Table~\ref{tab:hyperparams} records the default controller, retrieval, scheduling, and recovery settings.

\begin{table}[htbp]
\caption{Hyperparameters for Lagrangian, Chemical-CER, and Curriculum modules.}
\label{tab:hyperparams}
\centering
\small
\begin{tabular}{llc}
\toprule
State / operation & Parameter & Default \\
\midrule
\multirow{3}{*}{Constraint state}
  & Initial penalty $\rho_0$ & 1.0 \\
  & Growth factor $\gamma$ & 1.5 \\
  & Maximum penalty $\rho_{\max}$ & 100 \\
\midrule
\multirow{3}{*}{Experience state}
  & Retrieval count $k$ & 5 \\
  & Archive capacity & 1000 \\
  & Fingerprint radius / bits & 2 / 2048 \\
\midrule
\multirow{2}{*}{Schedule state}
  & Phase~1 end $R_1$ & 20 \\
  & Phase~2 end $R_2$ & 60 \\
\midrule
Context recovery & Stagnation trigger within each phase & 5 \\
\bottomrule
\end{tabular}
\end{table}

The initial $\rho_0$ sets the quadratic penalty when constraints activate or the active set expands.
The growth factor $\gamma$ increases this penalty after a violated accepted state, subject to the cap $\rho_{\max}$. Feasible accepted states leave it unchanged.
The retrieval count limits the number of examples exposed to the Generator, and archive capacity limits retained experience.
The curriculum changes only at its fixed round boundaries. The stagnation trigger controls proposal context and does not activate constraints.

%=============================================================================
\section{Broader Impact}
\label{app:impact}

\ours{} is a computational tool for molecular optimization and does not directly produce physical molecules.
Its outputs are SMILES strings that must be validated by medicinal chemists and tested in wet-lab assays before entering any drug development pipeline.
By incorporating ADMET constraints, the framework is designed to steer optimization toward molecules with better safety profiles, which could reduce the number of toxic candidates that proceed to animal testing.

A potential risk of any molecular generation system is the production of molecules with harmful properties (toxins, controlled substances, environmental pollutants).
The implemented screen penalizes predicted hERG inhibition. It does not include genotoxicity endpoints or establish safety against these broader risks.
Nonetheless, the system is not a substitute for regulatory toxicology assessment, and users should apply additional safety checks before any experimental synthesis.

The ADMET-PMO benchmark itself poses minimal risk, as it consists of computational oracle functions applied to virtual molecules.
Releasing the benchmark and evaluation code promotes transparency and reproducibility in the field.

%=============================================================================
\section{Comparison of Evaluated Baselines}
\label{app:concurrent}

Table~\ref{tab:concurrent} compares the original five baseline implementations along six axes. The online adaptations in Table~\ref{tab:admet_pmo} enable online measurements and constraint-aware selection or feedback. Their protocol is specified in Appendix~\ref{app:online}.  Its original-method roster matches Tables~\ref{tab:standard_pmo} and~\ref{tab:posthoc}. The comparison describes the configurations evaluated in this study.

\begin{table}[H]
\caption{Features enabled in the evaluated five-baseline protocol.  \ding{51} indicates enabled features and \ding{55} indicates features not enabled. The symbol \ding{109} denotes within-task memo or trajectory reuse without fingerprint-based experience retrieval.  Online ADMET columns concern the optimization loop. All five baselines separately receive the same post-hoc ADMET annotation.  $\dagger$ denotes our independent MiniMax-M3 reimplementations.}
\label{tab:concurrent}
\centering
\small
\resizebox{\linewidth}{!}{%
\begin{tabular}{lcccccc}
\toprule
System & Multi-agent & Online ADMET & Online ADMET & Memory & Cross-task & Five-task \\
 & & evaluation & constraints & retrieval & transfer & PMO \\
\midrule
Graph-GA & \ding{55} & \ding{55} & \ding{55} & \ding{55} & \ding{55} & \ding{51} \\
REINVENT4 & \ding{55} & \ding{55} & \ding{55} & \ding{55} & \ding{55} & \ding{51} \\
MOLLEO-MiniMax-M3$^{\dagger}$ & \ding{55} & \ding{55} & \ding{55} & \ding{55} & \ding{55} & \ding{51} \\
ExLLM-MiniMax-M3$^{\dagger}$ & \ding{55} & \ding{55} & \ding{55} & \ding{109} & \ding{55} & \ding{51} \\
SEISMO-MiniMax-M3 & \ding{55} & \ding{55} & \ding{55} & \ding{109} & \ding{55} & \ding{51} \\
\midrule
\ours{} & \ding{51} & \ding{51} & \ding{51} & \ding{51} & \ding{51} & \ding{51} \\
\bottomrule
\end{tabular}%
}
\end{table}

Graph-GA and REINVENT4 optimize the target-only objective in our baseline runs. REINVENT4's broader configurable scoring capabilities are not enabled as online ADMET constraints in this experiment.
MOLLEO-MiniMax-M3$^{\dagger}$ and ExLLM-MiniMax-M3$^{\dagger}$ are independent reimplementations.
SEISMO-MiniMax-M3 preserves the released prompts and original short-run PMO protocol with a MiniMax-M3 backend~\citep[Section~5.1 and Appendix~A]{kim2026seismo}.
The protocol permits at most 50 oracle evaluations. Our runs produce 27--48 unique target evaluations after molecular canonicalization.
After optimization ends, the final Top-10 score is held constant through the common 1,000-call AUC horizon, following the original paper's reporting convention.
This carry-forward is used only for metric integration and creates no additional oracle evaluations.
Post-hoc annotation evaluates the original generated trajectories through the common property protocol in Table~\ref{tab:posthoc}. PharmAgent uses property measurements during search to guide subsequent decisions.
Experience retrieval and cross-task transfer refer specifically to retrieval and reuse of stored optimization attempts. The main results quantify the performance advantage.

%=============================================================================
\begin{table}[H]
\caption{State and memory mechanisms in the five baseline frameworks and Chemical-CER.  Framework names omit the LLM backend. Implementation details are specified in the experimental setup.  Within-task state includes populations, learned policy parameters, and explicit history.  Structure-aware denotes fingerprint-based retrieval of optimization experiences. Cross-task denotes reuse of an optimization experience archive in the evaluated configurations.}
\label{tab:memory_compare}
\centering
\small
\begin{tabular}{llcc}
\toprule
Method & Within-task state & Structure-aware & Cross-task \\
\midrule
Graph-GA & Evolutionary population & \ding{55} & \ding{55} \\
REINVENT4 & Learned policy parameters & \ding{55} & \ding{55} \\
MOLLEO & Population / parent scores & \ding{55} & \ding{55} \\
ExLLM & Experience memo & \ding{55} & \ding{55} \\
SEISMO & Interaction trajectory & \ding{55} & \ding{55} \\
Chemical-CER (ours) & Retrieved transitions & \ding{51} & \ding{51} \\
\bottomrule
\end{tabular}
\end{table}

\paragraph{Complete aggregate results.}
Table~\ref{tab:posthoc} reports the original-protocol property comparison. Table~\ref{tab:final_scores} gives final target scores and uncertainty for both original and online configurations.
\begin{table}[H]
\caption{Original baselines with post-hoc property annotation, compared with online PharmAgent E2. Results average five tasks and five runs. The 5-Task Sum adds task-level AUC means. All metrics are higher-is-better. Baselines receive post-hoc ADMET annotation. \ours{} uses online constraints. Backend and evaluation protocols follow Section~\ref{sec:benchmark}. Best/second-best means are bold/underlined.}
\label{tab:posthoc}
\centering
\small
\begin{tabular}{lccc|cc}
\toprule
 & \multicolumn{3}{c|}{Target Score} & \multicolumn{2}{c}{Property-Aware Quality} \\
Method & Top-1 & AUC Top-10 & 5-Task Sum & AA-AUC & PHV \\
\midrule
Graph-GA & 0.6029 & 0.3998 & 1.9988 & 0.3367 & 0.5686 \\
REINVENT4 & 0.5864 & 0.3582 & 1.7911 & 0.2723 & 0.5238 \\
MOLLEO$^{\dagger}$ & \underline{0.7797} & \underline{0.5713} & \underline{2.8565} & \underline{0.4145} & \underline{0.7016} \\
ExLLM$^{\dagger}$ & 0.7364 & 0.5590 & 2.7950 & 0.3121 & 0.6065 \\
SEISMO & 0.6053 & 0.5613 & 2.8067 & 0.2994 & 0.5204 \\
\ours{} & \textbf{0.8514} & \textbf{0.7387} & \textbf{3.6936} & \textbf{0.7076} & \textbf{0.8248} \\
\bottomrule
\end{tabular}
\end{table}

\begin{table}[H]\centering\small
\caption{Final target-score summaries for all configurations across five fixed tasks and five runs. Means $\pm$ unadjusted 95\% interval half-widths. The lower block contains online adaptations. FF denotes feasibility-first and FP fixed-penalty search. Best/second-best means are bold/underlined across all rows. *SEISMO follows its original short-run protocol.}
\label{tab:final_scores}
\begin{tabular}{lcc}\toprule Method & Top-1 & Top-10 \\\midrule
PharmAgent E1 & $\mathbf{0.8641}\pm0.0645$ & $\mathbf{0.8606}\pm0.0634$ \\
PharmAgent E2 & $\underline{0.8514}\pm0.0561$ & $\underline{0.8305}\pm0.0751$ \\
Graph-GA & $0.6029\pm0.0194$ & $0.5498\pm0.0157$ \\
REINVENT4 & $0.5864\pm0.0469$ & $0.4613\pm0.0264$ \\
MOLLEO$^\dagger$ & $0.7797\pm0.0401$ & $0.7484\pm0.0406$ \\
ExLLM$^\dagger$ & $0.7364\pm0.0565$ & $0.7184\pm0.0561$ \\
SEISMO* & $0.6053\pm0.0512$ & $0.5733\pm0.0432$ \\
\midrule
Graph-GA (FF) & $0.4375\pm0.0141$ & $0.3219\pm0.0257$ \\
Graph-GA (FP) & $0.4837\pm0.0852$ & $0.3575\pm0.0981$ \\
REINVENT4 (FP) & $0.6221\pm0.0901$ & $0.4663\pm0.0197$ \\
MOLLEO$^\dagger$ (FF) & $0.6788\pm0.0517$ & $0.5864\pm0.0290$ \\
ExLLM$^\dagger$ (FF) & $0.6647\pm0.0321$ & $0.6352\pm0.0320$ \\
SEISMO* (FF) & $0.6405\pm0.0338$ & $0.5978\pm0.0261$ \\
\bottomrule\end{tabular}\end{table}

\FloatBarrier

\section{Seed Molecule Pool}
\label{app:seeds}

For seed selection (Section~\ref{sec:experiments}), we use a pool of 26 diverse molecules covering simple aromatics, heterocycles, approved drugs, bicyclic scaffolds, aliphatic chains, and common drug motifs:

{\small
\begin{verbatim}
c1ccccc1  c1ccncc1  c1ccoc1  c1ccsc1  c1cc[nH]c1
C1CCNCC1  C1CCOCC1  C1CCSCC1
CC(=O)Oc1ccccc1C(=O)O   CC(C)Cc1ccc(cc1)C(C)C(=O)O
CC(=O)Nc1ccc(O)cc1      CC(C)NCC(O)c1ccc(O)c(O)c1
c1ccc2ncccc2c1  c1ccc2c(c1)[nH]c1ccccc12  c1ccc2c(c1)cc1ccccc1n2
CCCCCCCC(=O)O  CC(C)(C)C(=O)O  C1CCC(CC1)O  CC(C)(C)O
O=C(O)c1ccccc1O  c1ccc(cc1)O  NC(=O)c1ccccc1
c1ccc(cc1)S(=O)(=O)N  CC(=O)c1ccccc1  C1COCCN1  C1CNCCN1
\end{verbatim}
}

At the start of each run, all 26 molecules are scored by the target oracle (at most 26 calls, $<$3\% of the 1,000-call budget). The highest-scoring molecule initializes optimization, reducing dependence on a single arbitrary seed at low oracle cost.

\clearpage
\section{Limitations and Remaining Validation}
\label{app:limitations}
\ours{} leads all five E2 aggregate metrics, with higher means in 30/30 target-AUC, 27/30 AA-AUC, and 29/30 PHV task-level comparisons (Figure~\ref{fig:online_task_margins}). The few online exceptions concern Sitagliptin and Albuterol, including a PHV reversal of 0.0004. In E1, SEISMO leads DRD2 and Albuterol. E2's target-AUC sum is 3.6936 versus E1's 3.9208, and their target-AUC and PHV difference intervals cross zero.

Evaluation covers five tasks and predicted properties. Experimental safety validation and cross-domain transfer remain open. All methods receive the same oracle allowance, with actual evaluations and computation varying by protocol. SEISMO retains short-run optimization and terminal-score carry-forward for AUC integration. Online adaptations were designed after the original PharmAgent runs and fixed before formal outcomes. Hosted-model versions and sampling are provider-controlled rather than fixed by local run labels. We plan to extend evaluation to open-weight language models as computational resources permit.

Proposition~\ref{prop:conditional_feasibility} assumes uniformly bounded multipliers, which is not established for the search and cannot be verified from finite runs. The analysis does not guarantee finite-budget feasibility, optimality, acceleration, or stationarity.

Intervals are unadjusted. Molecular cases describe fixed endpoints rather than isolated scheduling effects, and original-to-online contrasts compare complete configurations.

\end{document}